\documentclass[letterpaper]{article} 
\usepackage{aaai2026}  
\usepackage{times}  
\usepackage{helvet}  
\usepackage{courier}  
\usepackage[hyphens]{url}  
\usepackage{graphicx} 
\usepackage{natbib}  
\usepackage{caption} 
\usepackage{algorithm}
\usepackage{algorithmic}

\usepackage{newfloat}
\usepackage{listings}
\DeclareCaptionStyle{ruled}{labelfont=normalfont,labelsep=colon,strut=off} 
\floatstyle{ruled}
\newfloat{listing}{tb}{lst}{}
\floatname{listing}{Listing}
\usepackage{amsmath}
\usepackage{amssymb}
\usepackage{mathtools}
\usepackage{amsthm}
\usepackage{physics}
\usepackage{mathrsfs}
\mathtoolsset{showonlyrefs}
\usepackage{enumerate}
\usepackage{lastpage}
\usepackage{float}
\usepackage{verbatim}
\usepackage{microtype}
\usepackage{graphicx}
\usepackage{subfigure}
\usepackage{booktabs} 

\newcommand{\explain}[1]{&&\text{(#1)}}

\newcommand{\fS}{\mathcal{S}}
\newcommand{\fA}{\mathcal{A}}
\newcommand{\fY}{\mathcal{Y}}

\newcommand{\fF}{\mathcal{F}}
\newcommand{\fO}{\mathcal{O}}

\newcommand{\fX}{\mathcal{X}}
\newcommand{\fV}{\mathcal{V}}

\newcommand{\R}{\mathbb{R}}

\newcommand{\E}{\mathbb{E}}

\newcommand{\ns}{{|\fS|}}

\newcommand{\ny}{{|\fY|}}

\newcommand{\tb}[1]{{\textbf{#1}}}

\newcommand{\paren}[1]{{\left(#1\right)}}
\newcommand{\e}[1]{{\epsilon}^{\paren{#1}}}
\newcommand{\Mbb}{{\overline{\overline{M}}}}
\newcommand{\ebb}[1]{{\overline{\overline{\epsilon}}}^{\paren{#1}}}
\newcommand{\Mb}{{\overline{M}}}
\newcommand{\eb}[1]{{\overline{\epsilon}}^{\paren{#1}}}

\theoremstyle{plain}
\newtheorem{theorem}{Theorem}[section]

\newtheorem{lemma}[theorem]{Lemma}
\newtheorem{corollary}[theorem]{Corollary}
\theoremstyle{definition}

\newtheorem{assumption}[theorem]{Assumption}
\theoremstyle{remark}
\newtheorem{remark}[theorem]{Remark}

\title{Asymptotic and Finite Sample Analysis of \\Nonexpansive Stochastic Approximations with Markovian Noise}
\author {
    Ethan Blaser\textsuperscript{\rm 1}, 
    Shangtong Zhang\textsuperscript{\rm 1}
}
\affiliations {
    \textsuperscript{\rm 1} Department of Computer Science, University of Virginia\\
    blaser@email.virginia.edu, shangtong@virginia.edu
}

\begin{document}

\maketitle

\begin{abstract}
    Stochastic approximation is a powerful class of algorithms
    with celebrated success. However, a large body of previous analysis focuses on stochastic approximations driven by contractive operators,
    which is not applicable in some important reinforcement learning settings like the average reward setting. This work instead investigates stochastic approximations with merely nonexpansive operators. In particular, we study nonexpansive stochastic approximations with Markovian noise, providing both asymptotic and finite sample analysis.
    Key to our analysis are novel bounds of noise terms resulting from the Poisson equation. As an application, we prove for the first time that classical tabular average reward temporal difference learning converges to a sample-path dependent fixed point.
\end{abstract}


\section{Introduction}
Stochastic approximation (SA) algorithms \citep{robbins1951stochastic, kushner2003stochastic, borkar2009stochastic} form the foundation of many iterative optimization and learning methods by updating a vector incrementally and stochastically. Prominent examples include stochastic gradient descent \citep{kiefer1952stochastic} and temporal difference (TD) learning \citep{sutton1988learning}. These algorithms generate a sequence of iterates $\qty{x_n}$ starting from an initial point $x_0 \in \R^d$ through the recursive update:
\begin{equation}
    x_{n+1} \doteq x_n + \alpha_{n+1}\qty(H\qty(x_n, Y_{n+1})-x_n)\label{eq:sa}
    \tag{SA}
\end{equation}
where $\left\{\alpha_n\right\}$ is a sequence of learning rates, $\{Y_n\}$ is a sequence of random noise in a space $\fY$, and a function $H:\mathbb{R}^d \times \mathcal{Y} \rightarrow \mathbb{R}^d$ maps the current iterate $x_n$ and noise $Y_{n+1}$ to the actual incremental update. 
We use $h$ to denote the expected update, i.e., $h(x) \doteq \E\qty[H(x, y)]$,
where the expectation will be formally defined shortly.

Despite the foundational role of SA in analyzing reinforcement learning (RL, \citet{sutton2018reinforcement}) algorithms, most of the existing literature assumes that the expected mapping $h$ is a contraction.
However, in many problems in RL, particularly those involving average reward formulations \citep{tsitsiklis1999average, puterman2014markov, wan2020learning, wan2021average, he2022emphatic}, $h$ is only guaranteed to be non-expansive, not contractive. 
Table \ref{tab:comparison} highlights the relative scarcity of results concerning nonexpansive mappings. As a result, it is surprising that the convergence of some of the simplest and most fundamental RL algorithms, such as tabular average reward TD \citep{tsitsiklis1999average}, has not been fully settled, despite more than 25 years having passed since its introduction.

One tool for analyzing \eqref{eq:sa} with nonexpansive $h$,
which has recently gained renewed attention, is Krasnoselskii-Mann (KM) iterations:
\begin{align}
    \label{eq km update}
    x_{n+1} = x_n + \alpha_{n+1} (h(x_n) - x_n).
    \tag{KM}
\end{align}
Under some other restrictive conditions, 
\citet{krasnosel1955two} first proves the convergence of~\eqref{eq km update} to a fixed point of $h$ and this result is further generalized by \citet{Edelstein1966ARO,ishikawa1976fixed,reich1979weak,liu1995ishikawa}.
More recently, \citet{cominetti2014rate} use a novel fox-and-hare model to connect KM iterations with Bernoulli random variables, providing a sharper convergence rate for $\norm{x_n - h(x_n)} \to 0$.
\citet{kim2007robustness,cominetti2014rate,bravo2019rates} further consider \eqref{eq km update} with some deterministic additive noise.

However, practitioners usually do not have access to $h$ directly.
Instead,
they only have access to a noisy estimate of $h$ (cf. $H$ in~\eqref{eq:sa}).
As a result,
the general SA update~\eqref{eq:sa} is also called the Stochastic KM (SKM) iterations when $h$ is nonexpansive.
Under mild conditions,
\citet{bravo2024stochastic} prove the almost sure convergence of SKM,
together with the convergence rates of $\E\qty[\norm{x_n - h(x_n)}]$.
However,
one significant limitation of \citet{bravo2024stochastic} is that they assume $\qty{Y_t}$ are i.i.d.,
which significantly restricts their applications in RL because the corresponding $\qty{Y_t}$ in many RL algorithms (e.g., the aforementioned tabular average reward TD) is a Markov chain.
This is the second gap that this work shall close.


\begin{table*}[h!]
\centering
\caption{Overview of stochastic approximation methods, with a focus on those that consider non-expansive mappings. ``Non-expansive $h$" refers to works where the expected mapping is non-expansive, as opposed to strictly a contraction. ``Markovian $\qty{Y_n}$" indicates cases where the noise term $\qty{Y_n}$ is Markovian. ``Asymptotic" refers to works that prove almost sure convergence, which is not necessarily weaker than non-asymptotic convergence results. Note that we present only a representative subset of results for SA with contractive mappings due to an abundance of literature in the area. For a more comprehensive treatment, see \citet{DBLP:books/sp/BenvenisteMP90, kushner2003stochastic, borkar2009stochastic}.}
\label{tab:comparison}
\begin{tabular}{lcccc}
\specialrule{1.2pt}{0pt}{2pt}
\textbf{} & \textbf{Nonexpansive $h$} & \textbf{Markovian $\{Y_n\}$} & \textbf{Asymptotic} & \textbf{Non-Asymptotic} \\ 
\specialrule{1.2pt}{0pt}{2pt}
\citet{krasnosel1955two} & \checkmark &  & \checkmark &  \\ \hline
\citet{ishikawa1976fixed} & \checkmark &  & \checkmark &  \\ \hline
\citet{reich1979weak} & \checkmark &  & \checkmark &  \\ \hline
\citet{DBLP:books/sp/BenvenisteMP90} &  &  & \checkmark &  \\ \hline
\citet{liu1995ishikawa} &  &  & \checkmark &  \\ \hline
\citet{szepesvari1997asymptotic} &  &  & \checkmark &  \\ \hline
\citet{abounadi2002stochastic} & \checkmark &  & \checkmark &  \\ \hline
\citet{tadic2002almost} &  & \checkmark &  & \checkmark \\ \hline
\citet{kushner2003stochastic} &  &  & \checkmark &  \\ \hline
\citet{koval2003law} &  &  & \checkmark & \checkmark \\ \hline
\citet{tadic2004almost} &  & \checkmark &  & \checkmark \\ \hline
\citet{kim2007robustness} & \checkmark &  & \checkmark &  \\ \hline
\citet{borkar2009stochastic} &  &  & \checkmark &  \\ \hline
\citet{cominetti2014rate} & \checkmark &  & \checkmark & \checkmark \\ \hline
\citet{bravo2019rates} & \checkmark &  & \checkmark & \checkmark \\ \hline
\citet{chen2021lyapunov} &  & \checkmark &  & \checkmark \\ \hline
\citet{borkar2021ode} &  & \checkmark & \checkmark & \checkmark \\ \hline
\citet{karandikar2024convergence} &  & \checkmark & \checkmark & \checkmark \\ \hline
\citet{bravo2024stochastic} & \checkmark &  & \checkmark & \checkmark \\ \hline
\citet{qian2024almost} &  & \checkmark & \checkmark & \checkmark \\ \hline
\citet{liu2024ode} &  & \checkmark & \checkmark &  \\ \hline
\textbf{Ours} & \checkmark & \checkmark & \checkmark & \checkmark \\ \hline
\specialrule{1.2pt}{0pt}{2pt}
\end{tabular}
\end{table*}

To summarize, we make two contributions in this work to close the two gaps.
\tb{First}, Theorem \ref{thm:main_conv_result} proves that the sequence $\qty{x_n}$ generated by \eqref{eq:sa} with Markovian $\qty{Y_n}$ and nonexpansive $h$, converges almost surely to some random point $x_* \in \fX_*$, where $\fX_*$ is the set of fixed points of $h$. 
Importantly, $x_*$ may depend on the entire sample-path.
Theorem \ref{thm:conv_rate} further provides the convergence rate of the expected residuals $\E\qty[\norm{x_n - h(x_n)}]$. 
Both only assume $\qty{Y_t}$ is a Markov chain.
Table~\ref{tab:comparison} highlights the improvement of this work over those prior.
The key idea of our approach is to use Poisson's equation to decompose the error $\qty{H(x_n, Y_{n+1})-h(x_n)}$ into boundable error terms \citep{DBLP:books/sp/BenvenisteMP90}. 
While Poisson's equation has been previously used for handling Markovian noise,
our method departs from prior arts in how we bound the resulting error terms.
Specifically, \citet{DBLP:books/sp/BenvenisteMP90} and \citet{konda2000actor} use stopping times, 
while \citet{borkar2021ode} employ a Lyapunov function and use the scaled iterates technique. 
By contrast, we leverage a 1-Lipschitz continuity assumption on $H$ to directly control the growth of error terms.
\tb{Second},
Theorem \ref{thm:avg_rew_td} uses our novel SKM results to provide the first proof of almost sure convergence of tabular average reward TD to a possibly sample-path dependent fixed point.

\paragraph{Notations} In this paper, all vectors are column. We use $\norm{\cdot}$ to denote a generic operator norm.
We use $\norm{\cdot}_2$ and $\norm{\cdot}_\infty$ to denote $\ell_2$ norm and infinity norm respectively. We use $\fO(\cdot)$ to hide deterministic constants for simplifying presentation,
while the letter $\zeta$ is reserved for sample-path dependent constants.

\section{Asymptotic Analysis of SKM Iterations} \label{sec:SA}

To broaden the applicability of our result,
we future allow~\eqref{eq:sa} to have additional additive noise.
Namely,
we consider the following SKM updates 
\begin{equation} 
\label{eq:skm_markov}
    \!\! x_{n+1} = x_n + \alpha_{n+1}\left(H(x_n, Y_{n+1}) - x_n + \e{1}_{n+1}\right)  \tag{SKM},
\end{equation}
where 
$\qty{x_n}$ are stochastic vectors evolving in $\R^d$,
$\qty{Y_n}$ is a Markov chain evolving in a finite state space $\fY$,
$H: \R^d \times \fY \to \R^d$ defines the update,
$\qty{\e{1}_{n+1}}$ is a sequence of stochastic noise evolving in $\R^d$,
and $\qty{\alpha_n}$ is a sequence of deterministic learning rates. 
We make the following assumptions.
\begin{assumption}[Ergodicity]\label{as:steadystate}
The Markov chain $\qty{Y_n}$ is irreducible and aperiodic. 
\end{assumption}
The Markov chain $\qty{Y_n}$ thus adopts a unique invariant distribution,
denoted $d_\mu$.
We use $P$ to denote the transition matrix of $\qty{Y_n}$.
\begin{assumption}[1-Lipschitz] \label{as:H 1 Lipschitz}
    The function $H$ is 1-Lipschitz continuous in its first argument w.r.t. some operator norm $\norm{\cdot}$ and uniformly in its second argument, i.e., for any $x, x', y$, it holds that
    \begin{align}
        \norm{H(x, y) - H(x', y)} \leq \norm{x - x'}.
    \end{align}
\end{assumption}
This assumption has two important implications.
First, it implies that $H(x, y)$ can grow at most linearly.
Indeed, let $x' = 0$, we get
    $\norm{H(x, y)} \leq \norm{H(0, y)} + \norm{x}$.
Define $C_H \doteq \max_y \norm{H(0, y)}$,
we get
\begin{align}
    \norm{H(x, y)} \leq C_H + \norm{x}. \label{eq:H linear growth}
\end{align}
Second,
define
the function $h: \mathbb{R}^d \rightarrow \mathbb{R}^d$ as the expectation of $H$ over the stationary distribution $d_\mu$: \begin{align} h(x) \doteq \E_{y\sim d_\mu}[H(x, y)]. \label{def: h} \end{align}
We then have that $h$ is non-expansive. 
Namely,
\begin{align}
    \norm{h(x) - h(x')} &\textstyle \leq \sum_y d_\mu(y) \norm{H(x, y) - H(x', y)} \\
    &\leq \norm{x - x'}. \label{eq:h nonexpansive}
\end{align}
We need to assume that the problem is solvable.
\begin{assumption}[Fixed Points]\label{as: fixed points nonempty}
    The non-expansive operator $h$ adopts at least one fixed point.
\end{assumption}
We use $\fX_* \neq \emptyset$ to denote the set of fixed points of $h$.
\begin{assumption}[Learning Rate]\label{as:lr}
    The learning rate $\qty{\alpha_n}$ has the form
    \begin{align}
        \alpha_n = \frac{1}{(n+1)^b}, \alpha_0 = 0,
    \end{align}
    where $b \in (\frac{4}{5}, 1]$.
\end{assumption}
The primary motivation for requiring $b \in (\frac{4}{5}, 1]$ is that our learning rates $\alpha_n$ need to decrease quickly enough for certain key terms in the proof to be finite. The specific need for $b > \frac{4}{5}$ can be seen in the proof of \eqref{eq: ai aj^2 tj^2} in Lemma \ref{lem:lr}.
We now impose assumptions on the additive noise.
\begin{assumption}[Additive Noise]\label{as:e1}
\begin{align}
    \sum_{k=1}^\infty \alpha_k \norm{\e{1}_k} <& \infty \qq{a.s.,} \label{as: total noise} \\
    \textstyle \E\left[\norm{\e{1}_n}^2 \right] =& \textstyle \fO\qty(1 \, /\,n). \label{as: second moment}
\end{align}
\end{assumption}

The first part of Assumption \ref{as:e1} can be interpreted as a requirement that the total amount of additive noise remains finite. 
Additionally, we impose a condition on the second moment of this noise, requiring it to converge at the rate $\fO \qty(\frac{1}{n})$. 
While these assumptions on ${\e{1}_n}$ may seem restrictive, it should be noted that even if $\e{1}_n$ were absent, our work would still extend the results of \citet{bravo2024stochastic} to cases involving Markovian noise, as the Markovian noise component is already incorporated in $Y_n$, which represents a significant result.
For most RL applications involving algorithms which have only one set of learnable weights, the additional noise $\e{1}_k$ will simply be 0.
We are now ready to present the asymptotic analysis of~\eqref{eq:skm_markov}. 

\begin{theorem} \label{thm:main_conv_result}
Let Assumptions \ref{as:steadystate} - \ref{as:e1} hold.
Then the iterates $\qty{x_n}$ generated by~\\\eqref{eq:skm_markov}
satisfy 
\begin{align}
    \lim_{n\to\infty} x_n = x_* \qq{a.s.,}
\end{align}
where $x_* \in \fX_*$ is a possibly sample-path dependent fixed point.
Or more precisely speaking, let $\omega$ denote a sample path $(w_0, Y_0, Y_1, \dots)$ and write $x_n(\omega)$ to emphasize the dependence of $x_n$ on $\omega$.
Then there exists a set $\Omega$ of sample paths with $\Pr(\Omega) = 1$ such that for any $\omega \in \Omega$, the limit $\lim_{n\to\infty} x_n(\omega)$ exists, denoted as $x_*(\omega)$, and satisfies $x_*(\omega) \in \fX_*$.
\end{theorem}
\begin{proof}
We first define two useful shorthands,
\begin{align}
\alpha_{k,n} & \doteq \alpha_k \prod_{j=k+1}^n \paren{1-\alpha_j}, \, \alpha_{n,n} \doteq \alpha_n \label{eq:alpha_in_define}, \\
    \tau_n & \doteq \sum_{k=1}^n \alpha_k\paren{1-\alpha_k}. \label{eq:tau_n_def}
\end{align}
    We then start with a decomposition of the error $H(x, Y_{n+1}) - h(x)$ using Poisson's equation akin to \citet{metivier1987theoremes,DBLP:books/sp/BenvenisteMP90}.
    Namely,
    thanks to the finiteness of $\fY$,
    it is well known (see, e.g., Theorem 17.4.2 of \citet{meyn2012markov}
  or Theorem 8.2.6 of \citet{puterman2014markov}) that there exists a function $\nu(x, y): \R^d \times \fY \to \R^d$ such that
    \begin{align}
        H(x, y) - h(x) = \nu(x, y) - (P\nu)(x, y). \label{eq: Poisson decomp}
    \end{align}
    Here, we use
    $P\nu$ to denote the function $(x, y) \mapsto \sum_{y'} P(y, y')\nu(x, y')$.
    The error can then be decomposed as
    \begin{align}
        \label{eq poisson noise representation}
        &H(x, Y_{n+1}) - h(x) = M_{n+1} + \e{2}_{n+1} + \e{3}_{n+1},
    \end{align}
    where
    \begin{align}
 M_{n+1} &\doteq \nu(x_n,Y_{n+2}) - (P\nu)(x_n,Y_{n+1}), \label{eq:M_define} \\
 \e{2}_{n+1} &\doteq \nu\paren{x_n,Y_{n+1}} - \nu\paren{x_{n+1},Y_{n+2}}, \label{eq:e2_define} \\
 \e{3}_{n+1} &\doteq  \nu\paren{x_{n+1},Y_{n+2}} - \nu\paren{x_n,Y_{n+2}}. \label{eq:e3_define}
    \end{align}
Here $\qty{M_{n+1}}$ is a Martingale difference sequence.
We then use
\begin{align}
    \xi_{n+1} &\doteq \e{1}_{n+1}+ \e{2}_{n+1} + \e{3}_{n+1}, \label{eq:xi_define}
\end{align}
to denote all the non-Martingale noise, yielding
\begin{align}
        x_{n+1}&=\paren{1-\alpha_{n+1}}x_n + \alpha_{n+1}\paren{h\paren{x_n}+M_{n+1}+\xi_{n+1}}.
\end{align}
We now define an auxiliary sequence $\qty{U_n}$ to capture how the noise evolves
\begin{align}
    U_{n+1}\doteq \paren{1-\alpha_{n+1}}U_n + \alpha_{n+1}\paren{M_{n+1}+\xi_{n+1}}, \, U_0 \doteq 0. \label{eq:U_n_define}
\end{align} 
If we can prove that the total noise is well controlled in the following sense
\begin{align}
    \sum_{k=1}^\infty \alpha_k \norm{U_{k-1}} &< \infty \qq{a.s.,} \label{eq:sum_Uk_converges}\\
    \lim_{n\rightarrow \infty}\norm{U_n}&=0 \qq{a.s.,} \label{eq:limit_to_0} 
\end{align}
then a result from \citet{bravo2024stochastic} 
can be applied on each sample path to complete the almost sure convergence proof.
The remainder of the proof is dedicated to the verification of these two conditions.

Telescoping~\eqref{eq:U_n_define} yields
\begin{align}
    U_n =& \underbrace{\sum_{k=1}^n \alpha_{k,n} M_k}_{\Mb_n} + \underbrace{\sum_{k=1}^n \alpha_{k,n} \e{1}_k}_{\eb{1}_n}+ \\&\quad\underbrace{\sum_{k=1}^n \alpha_{k,n} \e{2}_k}_{\eb{2}_n} + \underbrace{\sum_{k=1}^n \alpha_{k,n} \e{3}_k}_{\eb{3}_n}. \label{eq:Un_norm_bar_bounds}
\end{align}
Then, we can upper-bound \eqref{eq:sum_Uk_converges} as
\begin{align}
    \sum_{k=1}^n \alpha_k \norm{U_{k-1}} &\leq  \underbrace{\sum_{k=1}^n \alpha_k \norm{\Mb_{k-1}}}_{\Mbb_n} + \underbrace{\sum_{k=1}^n \alpha_k \norm{\eb{1}_{k-1}}}_{\ebb{1}_n} \\ & + \underbrace{\sum_{k=1}^n \alpha_k \norm{\eb{2}_{k-1}}}_{\ebb{2}_n}  + \underbrace{\sum_{k=1}^n \alpha_k \norm{\eb{3}_{k-1}}}_{\ebb{3}_n}. \label{eq:sumUn_derivation}
\end{align}
Here we bound only $\ebb{2}_n$ to demonstrate the novelty of our approach to handling these error terms.
The almost sure bounds for $\Mbb_{n}, \ebb{1}_n,$ and $\ebb{3}_n$ are provided in Lemmas~\ref{lem:sup_M},~\ref{lem:sup_e1}, and~\ref{lem:sup_e3}  respectively.
Starting with the definition of $\eb{2}_n$ from \eqref{eq:Un_norm_bar_bounds}, and substituting the definition of $\e{2}_n$ from \eqref{eq:e2_define} we have,
    \begin{align}
        &\eb{2}_n \\
        &= -\sum_{k=1}^n \alpha_{k,n} \qty(\nu\qty(x_{k},Y_{k+1}) - \nu\qty(x_{k-1},Y_{k})),\\
        &= -\sum_{k=1}^n \alpha_{k,n}\nu\qty(x_{k},Y_{k+1}) - \alpha_{k-1,n}\nu\paren{x_{k-1},Y_{k}}  \\
        &\quad \quad+ \alpha_{k-1,n}\nu\paren{x_{k-1},Y_{k}} - \alpha_{k,n} \nu\paren{x_{k-1},Y_{k}}, \\
        &= -\alpha_{n}\nu\paren{x_{n},Y_{n+1}} - \sum_{k=1}^n\qty(\alpha_{k-1,n}-\alpha_{k,n})\nu \qty(x_{k-1},Y_k),
    \end{align}
    where the last equality holds because $\alpha_0\doteq0$ and $\alpha_{n,n} = \alpha_n$. Taking the norm gives
    \begin{align}
        \norm{\eb{2}_n}
        &\leq \alpha_n \norm{\nu \paren{x_{n},Y_{n+1}}} \\ &\quad + \sum_{k=1}^n \abs{\alpha_{k-1,n}-\alpha_{k,n} }\norm{\nu \paren{x_{k-1},Y_k}}, \label{eq: e2 bar stop}\\
        &\leq \zeta_{\ref{lem:v_norm}}(\alpha_n\tau_n + \sum_{k=1}^n \left|\alpha_{k-1,n}-\alpha_{k,n} \right|\tau_{k-1}), \\
        &\leq 2\zeta_{\ref{lem:v_norm}}\alpha_n \tau_n,
    \end{align}
where the second inequality holds by Lemma \ref{lem:v_norm} with $\zeta_{\ref{lem:v_norm}}$ denoting a sample-path dependent constant defined in Lemma~\ref{lem:v_norm}, and the last inequality holds because $\alpha_0 \doteq 0$, and that $\alpha_{i,n}$ and $\tau_i$ are monotonically increasing (Lemma \ref{lem:bravo b1}).

Then, from the definition of $\ebb{2}_n$ in \eqref{eq:sum_Uk_converges}, we have
    \begin{align}
        \ebb{2}_n=\sum_{k=1}^n \alpha_k \norm{\eb{2}_{k-1}} \leq  2\zeta_{\ref{lem:v_norm}} \sum_{k=1}^n \alpha_{k}^2 \tau_{k},
    \end{align} 
    where the inequality holds because $\alpha_0 \doteq 0$ and $\alpha_k$ is decreasing.
    Then, by Lemma \ref{lem:lr}, we have $\sup_n \sum_{k=1}^n \alpha_{k}^2 \tau_{k} < \infty$, which when combined with the monotone convergence theorem proves that $\lim_{n\rightarrow \infty} \ebb{2}_n < \infty$, verifying~\eqref{eq:sum_Uk_converges}.

We now verify~\eqref{eq:limit_to_0}.
This time, rewrite $U_n$ as
\begin{align}
    U_n &= -\sum_{k=1}^n \alpha_k U_{k-1}+\alpha_k\paren{M_k + \e{1}_k + \e{2}_k + \e{3}_k}.
\end{align}
Lemma~\ref{lem:martingale_series_bound}, Assumption \ref{as:e1}, and Lemmas~\ref{lem:ak_e2},~\ref{lem:ak_e3} prove that
  $\sup_n \norm{\sum_{k=1}^n\alpha_k M_k} < \infty$ and $\sup_n \norm{\sum_{k=1}^n \alpha_k \e{j}_k} < \infty$ for $j \in \qty{1,2,3}$ respectively.

Together with~\eqref{eq:Un_norm_bar_bounds},
this means that $\sup_n \norm{U_n} < \infty$.
In other words,
we have established the stability of~\eqref{eq:U_n_define}.
Then, it can be shown (Lemma~\ref{lem:xi_noise_convergence}), using an extension of Theorem 2.1 of \citet{borkar2009stochastic} (Lemma~\ref{thm:borkar_2_thm1}),
that $\qty{U_n}$ converges to the globally asymptotically stable equilibrium of the ODE
    $\dv{U(t)}{t} = -U(t)$,
which is 0.
This verifies~\eqref{eq:limit_to_0}.
Lemma~\ref{lem: applying bravo 2.1} then invokes a result from \citet{bravo2024stochastic} and completes the proof.
\end{proof}

\begin{remark}
    We want to highlight that the technical novelty of our work comes from two sources. The first is that while the use of Poisson's equation for handling Markovian noise is well-established, including the noise representation in~\eqref{eq poisson noise representation}, previous works with such error decomposition (e.g., \citet{DBLP:books/sp/BenvenisteMP90,konda2000actor,borkar2021ode}) usually only need to bound terms like ${\sum_k \alpha_k \e{1}_k}$.
    In contrast, our setup requires the bounding of additional terms such as $\eb{1}_n = \sum_k \alpha_{k, n} \e{1}_k$
    and $\ebb{1}_n = \sum_i \alpha_i \norm{\eb{1}_{k-1}}$
    that appear novel and more challenging. 
    Specifically,
    \citet{DBLP:books/sp/BenvenisteMP90,konda2000actor} consider the stopping time when $\norm{x_n}$ first exceeds some threshold.
    \citet{borkar2021ode} develop a contractive and recursive bound for $\norm{\nu(x_k, Y_{k+1})}$.
    Both are highly complicated and do not apply to our problem of bounding $\ebb{1}$.
    We instead leverage the 1-Lipschitzness of $H$ and use the sample-path dependent direct bound (cf. Lemma~\ref{lem:v_norm}) for $\norm{\nu(x_k, Y_{k+1})}$. 
    Second, our work extends Theorem 2.1 of \citet{borkar2009stochastic} by relaxing an assumption on the convergence of the deterministic noise term. Instead of requiring the noise to converge to 0, we only require a more mild condition on the asymptotic rate of change of this noise term. This extension, detailed in Appendix \ref{sec:borkar_ext}, has independent utility beyond this work. 
\end{remark}

\section{Finite Sample Analysis of SKM Iterations}\label{sec:c_rate}
The previous analysis not only guarantees the almost sure convergence of the iterates, but can also be used to obtain estimates of the expected fixed-point residuals.

\begin{theorem} \label{thm:conv_rate}
Consider the iteration (\ref{eq:skm_markov}) and let Assumptions \ref{as:steadystate} $-$ \ref{as:e1} hold. There exists a constant $C_{\ref{thm:conv_rate}}$ such that
\begin{align}
\E\left[\norm{x_n - h\paren{x_n}} \right] \leq \frac{C_{\ref{thm:conv_rate}}}{\sqrt{\tau_n}} = \begin{cases} 
      \fO \big (1/\sqrt{n^{1-b}}\big) \, \text{if} \ \frac{4}{5}<b<1, \\
      \fO \big(1/\sqrt{\log n}\big) \, \text{if} \  b=1.
   \end{cases}
\end{align}
\end{theorem}
\begin{proof}
Considering the sequence $z_n \doteq x_n - U_n$ we have,
\begin{align}
    \norm{x_n - h\qty(x_n)}&\leq \norm{z_n - h\qty(z_n)} + 2\norm{z_n - x_n} , \\
    &= \norm{z_n - h\qty(z_n)} + 2\norm{U_n}.
\end{align}
where the inequality holds due to the non-expansivity of $h$ as proven in \eqref{eq:h nonexpansive}.
Then, our proof of Theorem \ref{thm:main_conv_result} guarantees the conditions under which the $\qty{z_n}$ is bounded. Specifically, we proved in Lemma \ref{lem: applying bravo 2.1} that if $\sum_{n=1}^{\infty}\alpha_k \norm{U_{k-1}} < \infty$ \eqref{eq:sum_Uk_converges} and $\norm{U_n} \rightarrow 0$ \eqref{eq:limit_to_0} almost surely, 
then a result from \citet{bravo2024stochastic} (included as Lemma \ref{lem:bravo_2.1} for completeness)
can be invoked to bound $\norm{z_n -h(z_n)}$. 
Specifically,
by identifying $e_k= U_{k-1}$ in Lemma~\ref{lem:bravo_2.1},
we get
\begin{align}
    &\norm{x_n - h\paren{x_n}} \\
    &\leq  \zeta_{\ref{lem:bravo_2.1}} \sigma\qty(\tau_n) +\sum_{k=2}^n 2\alpha_k \sigma\paren{\tau_n - \tau_k}\norm{U_{k-1}} + 4\norm{U_n}. \label{eq: asymptotic rate}
\end{align}
for $\zeta_{\ref{lem:bravo_2.1}} = 2 \text{dist}(x_0, \fX_*)+\sum_{k=2}^\infty \alpha_k\norm{U_{k-1}}$. However, $\zeta_{\ref{lem:bravo_2.1}}$ is a sample-path dependent constant whose order is unknown, and the random sequence $\norm{U_n}$ may occasionally become very large. Therefore, we compute the non-asymptotic error bound of the expected residuals $\E\left[\norm{x_n - h(x_n)}\right]$, which gives,
\begin{align}
    &\E\qty[\norm{x_n - h\paren{x_n}}] \leq  \underbrace{\E \qty[\zeta_{\ref{lem:bravo_2.1}}] \sigma\qty(\tau_n)}_{R_1} \\ &\quad + \underbrace{\sum_{k=2}^n 2\alpha_k \sigma\qty(\tau_n - \tau_k)\E\qty[\norm{U_{k-1}}]}_{R_2} + \underbrace{4\E\qty[\norm{U_n}]}_{R_3}. \label{eq: expected rate}
\end{align}

Recalling that $\sigma(y) \doteq \min\qty{1, 1/\sqrt{\pi y}}$, we can see that if there exists a deterministic constant $C_{\ref{thm:conv_rate}}$ such that $\E \qty[\zeta_{\ref{lem:bravo_2.1}}]\leq C_{\ref{thm:conv_rate}}$, we obtain that $R_1 = \fO\qty(1 / \sqrt{\tau_n})$. Therefore, in order to prove the Theorem, it is sufficient to find such a constant $C_{\ref{thm:conv_rate}}$ such that $\E \qty[\zeta_{\ref{lem:bravo_2.1}}]\leq C_{\ref{thm:conv_rate}}$, and prove that $R_2$, and $R_3$ are also $\fO\qty(1 / \sqrt{\tau_n})$. 

We proceed by first upper-bounding $R_3$, i.e., $\E\qty[\norm{U_n}]$. 
Taking the expectation of \eqref{eq:Un_norm_bar_bounds}, we have,
\begin{align}
    &\E \qty[\norm{U_n}] \\
    \leq& \E\qty[\norm{\Mb_n}] + \E\qty[\norm{\eb{1}_n}]+\E\qty[\norm{\eb{2}_n}] + \E\qty[\norm{\eb{3}_n}] \\
    \leq& C_{\ref{lem: M rate}} \tau_n\sqrt{\alpha_{n+1}} + \sum_{i=1}^n \alpha_{i,n}\E\qty[\norm{\e{1}_i}]+C_{\ref{lem: e2 rate}}\alpha_n\tau_n \\
    &\quad +  C_{\ref{lem: e3 rate}} \alpha_n \sum_{i=1}^n \alpha_i \tau_i \quad \text{(Corollaries \ref{lem: M rate}, \ref{lem: e2 rate}, \ref{lem: e3 rate})}\\
    \doteq& \omega_n \label{eq: expected un bound}
\end{align}
It can be shown (Lemma \ref{lem: E Un bound order}) that $\omega_n = \fO(\tau_n\sqrt{\alpha_{n+1}})$, which is dominated by $1/\sqrt{\tau_n}$.

For $R_2$, 
Lemma \ref{lem: bravo combo 2.11 3.1} proves, similarly to Theorems 2.11 and 3.1 of \citet{bravo2024stochastic}, that $R_2=\fO\qty(1 / \sqrt{\tau_n})$.

For $R_1$.
We first observe that
\begin{align}
    \sum_{k=2}^\infty \alpha_k \E \qty[\norm{U_{k-1}}] \leq \sum_{k=2}^\infty \alpha_k \omega_{k-1} = \fO\qty(\sum_{k=2}^\infty \alpha_k^{3/2}\tau_{k-1}),
\end{align}
which is finite by Lemma \ref{lem:lr}.
It is then obvious to see that there exists a $C_{\ref{thm:conv_rate}}$ such that $\E \qty[\zeta_{\ref{lem:bravo_2.1}}]\! = 2\text{dist}(x_0, \fX_*) + \sum_{k=2}^\infty \alpha_k \E \qty[\norm{U_{k-1}}] \leq C_{\ref{thm:conv_rate}}$,
which completes the proof.

 
\end{proof}

\begin{remark}
    While the convergence rate is relatively slow, especially compared to the discounted setting (e.g., \citet{chen2021lyapunov}), it matches the rate in the i.i.d. noise case for nonexpansive operators \citep{bravo2024stochastic}. This slow rate is inherent due to the nonexpansive nature of $h$ \citep{cominetti2014rate} and is not a limitation of our analysis.
\end{remark}

\section{Application in Average Reward Temporal Difference Learning}\label{sec:RL}
In this section, we provide the first proof of almost sure convergence to a fixed point for average reward TD in its simplest tabular form. Remarkably, this convergence result has remained unproven for over 25 years despite the algorithm’s fundamental importance and simplicity.  

\subsection{Reinforcement Learning Background}
In reinforcement learning (RL), we consider a Markov Decision Process (MDP; \citet{bellman1957markovian,puterman2014markov}) with a finite state space $\fS$,
a finite action space $\fA$,
a reward function $r: \fS \times \fA \to \R$,
a transition function $p: \fS \times \fS \times \fA \to [0, 1]$,
an initial distribution $p_0: \fS \to [0, 1]$.
At time step $0$,
an initial state $S_0$ is sampled from $p_0$.
At time $t$,
given the state $S_t$,
the agent samples an action $A_t \sim \pi(\cdot | S_t)$, 
where $\pi: \fA \times \fS \to [0, 1]$ is the policy being followed by the agent.
A reward $R_{t+1} \doteq r(S_t, A_t)$ is then emitted and the agent proceeds to a successor state $S_{t+1} \sim p(\cdot | S_t, A_t)$. 
In the rest of the paper, we will assume the Markov chain $\qty{S_t}$ induced by the policy $\pi$ is irreducible and thus adopts a unique stationary distribution $d_\mu$.
The average reward (a.k.a. gain, \citet{puterman2014markov}) is defined as
$\bar{J}_{\pi} \doteq \lim_{T\rightarrow \infty} \frac{1}{T}\sum_{t=1}^T \E\left[R_t\right].$
Correspondingly,
the differential value function (a.k.a. bias, \citet{puterman2014markov}) is defined as 
\begin{align}
     v_\pi(s) \doteq \lim_{T\to\infty} \frac{1}{T} \sum_{\tau=1}^T \E\left[ \sum_{i=1}^{\tau} (R_{t+i} - \bar{J}_{\pi}) \mid S_t = s\right].
\end{align}
The corresponding Bellman equation (a.k.a. Poisson's equation) is then
\begin{align}
    \label{eq bellman equation}
    v = r_\pi -\bar{J}_{\pi}e + P_\pi v ,
\end{align} 
where $v \in \R^\ns$ is the free variable,
$e$ denotes an all-one vector,
$r_\pi \in \R^\ns$ is the reward vector induced by the policy $\pi$, i.e., $r_\pi(s) \doteq \sum_a \pi(a|s) r(s, a)$, and $P_\pi \in \R^{\ns \times \ns}$ is the transition matrix induced by the policy $\pi$, i.e., $P_\pi(s, s') \doteq \pi(a|s)p(s'|s, a)$. 
It is known \citep{puterman2014markov} that all solutions to~\eqref{eq bellman equation} form a set 
\begin{align}
    \label{eq fixed points}
    \fV_* \doteq \qty{v_\pi + ce \mid c \in \R}.
\end{align}

The policy evaluation problem in average reward MDPs is to estimate $v_\pi$,
perhaps up to a constant offset $ce$.

\subsection{Average Reward Temporal Difference Learning}
\label{sec:avg reward td}  
Temporal Difference learning (TD; \citet{sutton1988learning}) is a foundational algorithm in RL \citep{sutton2018reinforcement}. Inspired by its success in the discounted setting, \citet{tsitsiklis1999average} proposed using the update rule \eqref{eq: avg td update} to estimate $v_\pi$ (up to a constant offset) for average reward MDPs. The updates are given by:  
\begin{align}
    \label{eq: avg td update}
    J_{t+1}\!&= J_t + \beta_{t+1} (R_{t+1} - J_t), \tag{Average Reward TD} \\ 
    v_{t+1}(S_t)\!&= v_t(S_t)\!+\! \alpha_{t+1}\!(R_{t+1} - J_t + v_t(S_{t+1}) - v_t(S_t)),
\end{align}  
where $ \{S_0, R_1, S_1, \dots\}$ is a trajectory of states and rewards from an MDP under a fixed policy in a finite state space $\mathcal{S}$, $J_t \in \mathbb{R}$ is the scalar estimate of the average reward $\bar J_\pi $, $ v_t \in \mathbb{R}^{|\mathcal{S}|} $ is the tabular value estimate, and $\{\alpha_t, \beta_t\}$ are learning rates.

To utilize Theorem \ref{thm:main_conv_result} to prove the almost sure convergence of ~\eqref{eq: avg td update},
we first rewrite it in a compact form to match that of \eqref{eq:skm_markov}. Define the augmented Markov chain $Y_{t+1} \doteq (S_t, A_t, S_{t+1})$.
It is easy to see that $\qty{Y_t}$ evolves in the finite space
$\fY \doteq \qty{(s, a, s') \mid \pi(a|s) > 0, p(s'|s, a) > 0}$.
We then define a function $H: \R^\ns \times \fY \to \R^\ns$ by defining the $s$-th element of $H(v, (s_0, a_0, s_1))$ as
\begin{align}
    &H(v, (s_0, a_0, s_1))[s] \doteq \\ &\quad \mathbb{I}_{\qty{s = s_0}} (r(s_0, a_0) - \bar{J}_{\pi}
    + v(s_1) - v(s_0))+ v(s).
\end{align}
Then, the update to $\qty{v_t}$ in~\eqref{eq: avg td update} can then be expressed as
\begin{align}
    \label{eq avg td update compact}
    v_{t+1} = v_t + \alpha_{t+1} \qty(H(v_t, Y_{t+1}) - v_t + \epsilon_{t+1}).
\end{align}
Here, $\epsilon_{t+1} \in \R^\ns$ is the random noise vector defined as
$\epsilon_{t+1}(s) \doteq \mathbb{I}\qty{s = S_t} (J_t - \bar J_\pi)$.
This $\epsilon_{t+1}$ is the current estimate error of the average reward estimator $J_t$.
Intuitively, the indicator $\mathbb{I}\qty{s = S_t}$ reflects the asynchronous nature of~\eqref{eq: avg td update}.
For each $t$, 
only the $S_t$-indexed element in $v_t$ is updated.

Throughout the rest of the section, we utilize the following assumption.
\begin{assumption}[Ergodicity]
    \label{as:rl_chain}
    Both $\fS$ and $\fA$ are finite.
    The Markov chain $\qty{S_t}$ induced by the policy $\pi$ is aperiodic and irreducible.
\end{assumption}
\begin{theorem} \label{thm:avg_rew_td}
Let Assumption \ref{as:rl_chain} hold. 
Consider the learning rates in the form of $\alpha_t = \frac{1}{(t+1)^b}, \beta_t = \frac{1}{t}$ with $b \in (\frac{4}{5}, 1]$.
Then the iterates $\qty{v_t}$ generated by~\eqref{eq: avg td update} satisfy
\begin{align}
    \lim_{t\to\infty}{v_t} = v_* \qq{a.s.,}
\end{align}
where $v_* \in \fV_*$ is a possibly sample-path dependent fixed point.
\end{theorem}
\begin{proof}
    We proceed via verifying assumptions of Theorem~\ref{thm:main_conv_result}. 
    In particular,
    we consider the compact form~\eqref{eq avg td update compact}.

Under Assumption~\ref{as:rl_chain}, it is obvious that $\qty{Y_t}$ is irreducible and aperiodic and adopts a unique stationary distribution.
    
    To verify Assumption \ref{as:H 1 Lipschitz}, we demonstrate that $H$ is $1-$Lipschitz in $v$ w.r.t $\norm{\cdot}_\infty$. For notation simplicity, let $y=(s_0, a_0, s_1)$. Separating by cases based on $s$, we have
    \begin{align}
    \abs{H(v,y)[s] - H(v',y)[s]}
    &=
    \begin{cases}
    \abs{v(s)-v'(s)}, & s\neq s_0,\\
    \abs{v(s_1)-v'(s_1)}, & s=s_0,
    \end{cases}
    \end{align}
    and in both cases the right side is at most $\|v - v'\|_\infty$. Thus,
    \begin{align}
    \norm{H(v,y)-H(v',y)}_\infty
    &= \max_{s\in\fS} \abs{H(v,y)[s]-H(v',y)[s]} \\
    &\le \norm{v - v'}_\infty.
    \end{align}
    It is well known that the set of solutions to Poisson's equation $\fV_*$ defined in \eqref{eq fixed points} is non-empty \citep{puterman2014markov}, verifying Assumption \ref{as: fixed points nonempty}.
    Assumption \ref{as:lr} is directly met by the definition of $\alpha_t$.
    
    To verify Assumption \ref{as:e1}, we first notice that for~\eqref{eq: avg td update}, we have
        $\norm{\e{1}_t}_\infty = \abs{\bar{J}_{\pi} - J_t}$.
    It is well-known from the ergodic theorem that $J_t$ converges to $\bar J_\pi$ almost surely.
    Assumption~\ref{as:e1}, however, requires both an almost sure convergence rate and an $L^2$ convergence rate.
    To this end, we rewrite the update of $\qty{J_t}$ as
    \begin{align}
        J_{t+1} = J_t + \beta_{t+1} \left(R_{t+1} + \gamma J_t \phi(S_{t+1}) - J_t \phi(S_t) \right) \phi(S_t),
    \end{align}
    where we define $\gamma \doteq 0$ and $\phi(s) \doteq 1 \, \forall s$.
    It is now clear that the update of $\qty{J_t}$ is a special case of linear TD in the discounted setting \citep{sutton1988learning}.
    Given our choice of $\beta_t = \frac{1}{t}$,
    the general result about the almost sure convergence rate of linear TD (Theorem~1 of \citet{tadic2002almost}) ensures that
    \begin{align}
        \abs{J_t - \bar J_\pi} \leq  \frac{\zeta_{\ref{thm:avg_rew_td}} \sqrt{\ln\ln t}}{\sqrt{t}} \qq{a.s.,}
    \end{align}
    where $\zeta_{\ref{thm:avg_rew_td}}$ is a sample-path dependent constant.
    This immediately verifies~\eqref{as: total noise}.
    We do note that this almost sure convergence rate can also be obtained via a law of the iterated logarithm for Markov chains (Theorem 17.0.1 of \citet{meyn2012markov}).
    The general result about the $L^2$ convergence rate of linear TD (Theorem~11 of \citet{srikant2019finite}) ensures that
    \begin{align}
        \E\qty[\abs{J_t - \bar J_\pi}^2] = \textstyle \fO\qty(\frac{1}{t}).
    \end{align}
    This immediately verifies~\eqref{as: second moment} and completes the proof.
    \end{proof}

\begin{remark}
    The convergence rate we established in Theorem~\ref{thm:conv_rate} also applies directly to the update in~\eqref{eq: avg td update}, and yields a bound on the expected residuals. However, this rate does not improve upon the existing result in \citet{zhang2021finite}, and thus we omit it here. A further discussion on the significance of Theorem \ref{thm:avg_rew_td} in comparison to the results in \citet{zhang2021finite} is deferred to the subsequent section.
\end{remark}

\subsection{Significance of Theorem 4.2}

Since \eqref{eq: avg td update} has been previously studied, we highlight the significance of Theorem~\ref{thm:avg_rew_td}, which provides the first proof of almost sure convergence of \eqref{eq: avg td update} to a (possibly sample-path dependent) fixed point in the tabular setting.

\citet{tsitsiklis1999average} proves the almost sure convergence for linear function approximation, where $v(s)$ is approximated by $\phi(s)^\top w$ with feature matrix $\Phi \in \R^{\ns \times K}$. This setting reduces to the tabular case when $\Phi = I$. However, their result requires assumptions like linear independence of $\Phi$'s columns and $\Phi w \neq c e$ for any scalar $c$.
The latter unfortunately does not hold in the tabular case (e.g., $I e = e$). 
With a non-trivial construction of $\Phi$, 
it is possible to adapt their result to show that the $\qty{v_t}$ in~\eqref{eq: avg td update} converge almost surely to some (possibly sample-path dependent) subset of $\fV_*$.
Even so, it is not clear whether $\qty{v_t}$ itself converges.
It is possible that $\qty{v_t}$ oscillates inside or around $\fV_*$.
Our result rules out this possibility by showing that on every sample-path $\qty{v_t}$ must converge to a single fixed point, although different sample-paths may converge to different fixed points.

\citet{zhang2021finite} later established $L^2$ convergence for the linear case without requiring $\Phi w \neq c e$, and derived convergence rates. However, $L^2$ convergence does not imply almost sure convergence, and even if one could strengthen their result to almost sure convergence, it would still only guarantee convergence to a set rather than a fixed point.

\citet{chen2025non} studies average reward TD using a seminorm contraction argument and show that the seminorm distance of the iterates to the fixed point set converges to zero. This does not imply convergence of the iterates themselves, since distinct points can have zero seminorm distance, so oscillations within $\fV_*$ are not ruled out. Theorem~\ref{thm:avg_rew_td} provides a stronger result by proving almost sure convergence of the iterates to a fixed point.

\section{Related Work}\label{sec: related}
\paragraph{ODE and Lyapunov Methods for Asymptotic Convergence} 
A large body of research has employed ODE-based methods to establish almost sure convergence of SA algorithms \citep{DBLP:books/sp/BenvenisteMP90, kushner2003stochastic, borkar2009stochastic}. These methods typically begin by proving the stability of the iterates $\qty{x_n}$ (i.e., $\sup_n \norm{x_n} < \infty$). \citet{abounadi2002stochastic} use this ODE method to study the convergence of SKM iterations, but they require the additive noise sequence to be uniformly bounded, and that the set of fixed points of the nonexpansive map be a singleton to prove the stability of the iterates.

The ODE@$\infty$ technique \citep{borkar2000ode,borkar2021ode,meyn2024projected,liu2024ode} is a powerful stability technique in RL. 
If the so-called ``ODE@$\infty$ is globally asymptotically stable,
existing results such as \citet{meyn2022control, borkar2021ode,liu2024ode} can be used to establish the desired stability of $\qty{x_t}$. However, if we consider a generic non-expansive operator $h$ which may admit multiple fixed points or induce oscillatory behavior, we cannot guarantee the global asymptotic stability of the ODE@$\infty$ without additional assumptions. This limits the utility of the ODE@$\infty$ method in analyzing \eqref{eq:skm_markov}. 

In addition to ODE methods, there are other works that use Lyapunov methods such as 
\citep{bertsekas1996neuro,konda2000actor,srikant2019finite,borkar2021ode,chen2021lyapunov,DBLP:journals/corr/abs-2111-02997,zhang2023convergence} to provide asymptotic and non-asymptotic results of various RL algorithms. Both the ODE and Lyapunov based methods are distinct from the fox-and-hare based approach for \eqref{eq km update} with additive noise introduced by \citep{cominetti2014rate} upon which our work is built.

\paragraph{Average Reward RL} 
The~\eqref{eq: avg td update} algorithm introduced by \citet{tsitsiklis1999average} is the most fundamental policy evaluation algorithm in average reward settings. In addition to the tabular setting we study here, \eqref{eq: avg td update} has also been extended to linear function approximation \citep{tsitsiklis1999average,konda2000actor,wu2020finite,zhang2021finite}.  

Furthermore, the~\eqref{eq: avg td update} algorithm has inspired the design of many other TD algorithms for average reward MDPs,
for both policy evaluation and control,
including \citet{konda2000actor, yang2016efficient, wan2021average, zhang2021policy, wan2020learning, he2022emphatic, saxena2023off}.
Because the operators in the average reward setting are not contractive, we envision that our work will shed light on the almost sure convergence of these follow-up algorithms.

\section{Conclusion}
In this work, we provide the first proof of almost sure convergence as well as non-asymptotic finite sample analysis of stochastic approximations under nonexpansive maps with Markovian noise. As an application, we provide the first proof of almost sure convergence of \eqref{eq: avg td update} to a potentially sample-path dependent fixed point. This result highlights the underappreciated strength of SKM iterations, a tool whose potential is often overlooked in the RL community. Addressing several follow-up questions could open the door to proving the convergence of many other RL algorithms. Do SKM iterations converge in $L^p$? Do they follow a central limit theorem or a law of the iterated logarithm? Can they be extended to two-timescale settings? Resolving these questions could pave the way for significant advancements in RL theory. We leave them for future investigation.

\section{Acknowledgments}
EB acknowledges support from the NSF Graduate Research Fellowship (NSF-GRFP) under award 1842490. This work is also supported in part by the US National Science Foundation under the awards III-2128019, SLES-2331904, and CAREER-2442098, the Commonwealth Cyber Initiative's Central Virginia Node under the award VV-1Q26-001, a Cisco Faculty Research Award, and an Nvidia academic grant program award.

\bibliography{bibliography}
\clearpage

\appendix
\onecolumn

\section{Mathematical Background}\label{app: math background}

\begin{lemma}[Theorem 2.1 from \citet{bravo2024stochastic}] \label{lem:bravo_2.1}
Let $\qty{z_n}$ be a sequence generated by~\eqref{eq inexact km update}. 
\begin{align}
    \label{eq inexact km update}
    z_{t+1} = z_t + \alpha_{t+1} (T z_t - z_t + e_{t+1}),
    \tag{IKM}
\end{align}
Let $\text{Fix}(T)$ denote the set of fixed points of $T$ (assumed to be nonempty).
Additionally, let $\tau_n$ be defined according to \eqref{eq:tau_n_def} and the real function $\sigma: (0,\infty) \rightarrow (0,\infty)$ as
\begin{align}
    \sigma(y) = \min\qty{{1,1 / \sqrt{\pi y}}}.
\end{align}

If $\zeta_{\ref{lem:bravo_2.1}} \geq 0$ is such that $\norm{Tz_n -x_0} \leq \zeta_{\ref{lem:bravo_2.1}}$  for all $n\geq1$, then
    \begin{equation}
        \norm{z_n - Tz_n} \leq \zeta_{\ref{lem:bravo_2.1}} \sigma\paren{\tau_n} + \sum_{k=1}^n 2\alpha_k\norm{e_k}\sigma\paren{\tau_n-\tau_k} + 2 \norm{e_{n+1}}. \label{eq:bravo_2.1}
    \end{equation}
Moreover, if $\tau_n\rightarrow \infty$ and $\norm{e_{n}}\rightarrow 0$ with $S\doteq \sum_{n=1}^\infty \alpha_n \norm{e_n} < \infty$, then \eqref{eq:bravo_2.1} holds with $\zeta_{\ref{lem:bravo_2.1}} = 2\inf_{x \in \text{Fix}(T)}\norm{x_0 - x}+S$, and we have $\norm{z_n - Tz_n} \rightarrow 0$ as well as $z_n \rightarrow x_*$ for some fixed point $x_*\in \text{Fix}(T)$
\end{lemma}

\begin{lemma}[Monotonicity of $\alpha_{k,n}$ from Lemma B.1 in \citet{bravo2024stochastic}] \label{lem:bravo b1}
    For $\alpha_n = \frac{1}{\paren{n+1}^b}$ with $0<b\leq1$ and $\alpha_{i,n}$ in \eqref{eq:alpha_in_define}, we have $\alpha_{k,n} \leq \alpha_{k+1,n}$ for $k\geq 1$ so that $\alpha_{k+1,n}\leq \alpha_{n,n} = \alpha_n$.
\end{lemma}

\begin{lemma}[Lemma B.2 from \cite{bravo2024stochastic}] \label{lem:bravo_b.2}
    For $\alpha_n = \frac{1}{\paren{n+1}^b}$ with $0<b\leq1$ and $\alpha_{i,n}$ in \eqref{eq:alpha_in_define}, we have $\sum_{k=1}^n \alpha_{k,n}^2 \leq \alpha_{n+1}$ for all $n\geq 1$.
\end{lemma}

\begin{lemma}[Monotone Convergence Theorem from \citet{folland1999real}]\label{lem: beppo levi}
    Given a measure space $\qty(X, M, \mu)$, define $L^+$ as the space of all measurable functions from $X$ to $[0,\infty]$. Then, if $\qty{f_n}$ is a sequence in $L^+$ such that $f_j \leq f_{j+1}$ for all j, and $f= \lim_{n\rightarrow \infty} f_n$, then $\int f d\mu = \lim_{n\rightarrow \infty} \int f_n d\mu$.
\end{lemma}

\section{Additional Lemmas from Section \ref{sec:SA}}
In this section, we present and prove the lemmas referenced in Section \ref{sec:SA} as part of the proof of Theorem \ref{thm:main_conv_result}. Additionally, we establish several auxiliary lemmas necessary for these proofs.

We begin by proving several convergence results related to the learning rates.
\begin{lemma}[Learning Rates]\label{lem:lr}
With $\tau_n$ defined in \eqref{eq:tau_n_def} we have,
\begin{align}
    \tau_n = \begin{cases} 
      \fO\qty(n^{1-b}) & \text{if} \quad \frac{4}{5}<b<1, \\
      \fO\qty(\log n) & \text{if} \quad b=1.
      \end{cases} \label{eq: tau_n order}
\end{align}
This further implies,
\begin{align}
    \sup_n \, \sum_{k=1}^n \alpha_k^2 \tau_k &< \infty, \label{eq: ai2 ti}\\ 
    \sup_n \, \sum_{k=1}^n \alpha_k^2 \tau_k^2 &< \infty, \label{eq: ai2 ti2}\\
    \sup_n \, \sum_{k=1}^n \alpha_k^{3/2}\tau_{k-1}&<\infty, 
    \label{eq: ai32 ti}\\
    \sup_n \sum_{k=0}^{n-1}\abs{\alpha_k - \alpha_{k+1}}\tau_k &< \infty, \label{eq: abs ai ti}\\
    \sup_n \space \sum_{k=1}^n \alpha_k^2 \sum_{j=1}^{i-1}\alpha_j\tau_j &< \infty, \label{eq: ai2 sum aj tj}\\
    \sup_n\sum_{k=1}^n \alpha_k \sqrt{\sum_{j=1}^{k-1} \alpha_{j,k-1}^2 \tau_{j-1}^2 } &< \infty, \label{eq: ai aj^2 tj^2}\\
\end{align}
\end{lemma}
Since this Lemma is comprised of several short proofs regarding the deterministic learning rates defined in Assumption \ref{as:lr}, we will decompose each result into subsections. Recall that $\alpha_n \doteq \frac{1}{\qty(n+1)^b}$ where $\frac{4}{5}<b\leq 1$.

\paragraph{\eqref{eq: tau_n order}:}
\
\begin{proof}
From the definition of $\tau_n$ in \eqref{eq:tau_n_def}, we have
\begin{align}
    \tau_n &\doteq \sum_{k=1}^{n}\alpha_k\qty(1-\alpha_k)\leq \sum_{k=1}^{n}\alpha_k = \sum_{k=1}^{n} \frac{1}{(k+1)^b}.
\end{align}

Case 1: $b=1$. It is easy to see $\tau_n = \fO\qty(\log n)$.

Case 2: When $b<1$, we can approximate the sum with an integral, with
\begin{align}
    \sum_{k=1}^{n}\frac{1}{\qty(k+1)^b} \leq \int_{1}^n \frac{1}{k^b} \,dk = \frac{n^{1-b} - 1}{1-b}
\end{align}
Therefore we have $\tau_n = \fO \qty( n^{1-b})$ when $b<1$.
\end{proof}

In analyzing the subsequent equations, we will use the fact that $\tau_n = \fO \qty(\log n)$ when $b=1$ and $\tau_n = \fO \qty( n^{1-b})$ when $\frac{4}{5}<b<1$.  Additionally, we have $\alpha_n = \qty( \frac{1}{n^b})$. 
\paragraph{\eqref{eq: ai2 ti}:}
\
\begin{proof}
We have an order-wise approximation of the sum
\begin{align}
\sum_{k=1}^n \alpha_k^2 \tau_k = \begin{dcases} 
      \fO \qty(\sum_{k=1}^n \frac{1}{k^{3b-1}}) & \text{if} \ \frac{4}{5}<b<1, \\
      \fO \qty(\sum_{k=1}^n \frac{\log(k)}{k^{2}}) & \text{if} \ b=1.
   \end{dcases}.
\end{align}
In both cases of $b=1$ and $\frac{4}{5}<b<1$, the series clearly converge as $n\rightarrow \infty$.  
\end{proof}

\paragraph{\eqref{eq: ai32 ti}:}
\
\begin{proof}
We have an order-wise approximation of the sum
\begin{align}
\sum_{k=1}^n \alpha_k^{3/2} \tau_k = \begin{dcases} 
      \fO \qty(\sum_{k=1}^n \frac{1}{k^{\frac{5}{2}b-1}}) & \text{if} \ \frac{4}{5}<b<1, \\
      \fO \qty(\sum_{k=1}^n \frac{\log(k)}{k^{3/2}}) & \text{if} \ b=1.
   \end{dcases}.
\end{align}
In both cases of $b=1$ and $\frac{4}{5}<b<1$, the series clearly converge as $n\rightarrow \infty$.  
\end{proof}

\paragraph{\eqref{eq: ai2 ti2}:}
\
\begin{proof}
We can give an order-wise approximation of the sum
\begin{align}
\sum_{k=1}^n \alpha_k^2 \tau_k^2 = \begin{dcases} 
      \fO \qty(\sum_{k=1}^n \frac{1}{k^{4b-2}}) & \text{if} \ \frac{4}{5}<b<1, \\
      \fO \qty(\sum_{k=1}^n \frac{\log^2(k)}{k^{2}}) & \text{if} \ b=1.
   \end{dcases}.
\end{align}
In both cases of $b=1$ and $\frac{4}{5}<b<1$, the series clearly converge as $n\rightarrow \infty$. 
\end{proof}
\paragraph{\eqref{eq: abs ai ti}:}
\
\begin{proof}
Since $\alpha_n$ is strictly decreasing, we have $\abs{\alpha_k - \alpha_{k+1}} = \alpha_k - \alpha_{k+1}$.

Case 1: For the case where $b=1$, it is trivial to see that,
\begin{align}
    \sum_{k=1}^n \abs{\alpha_k - \alpha_{k+1}}\tau_k = \fO \qty(\sum_{k=1}^n \frac{\log(k)}{k^{2}+k}).
\end{align}
This series clearly converges.

Case 2:
For the case where $\frac{4}{5}< b < 1$, we have
\begin{align}
    \alpha_n - \alpha_{n+1} &= \fO \qty(\frac{1}{n^b} - \frac{1}{(n+1)^b}),\\
    &= \fO \qty(\frac{(n+1)^b - n^b}{n^b(n+1)^b}). \label{eq: nb frac}
\end{align}
To analyze the behavior of this term for large $n$ we first consider the binomial expansion of $(n+1)^b$,
\begin{align}
    (n+1)^b &= n^b \qty(1+\frac{1}{n})^b = n^b(1+b\frac{1}{n} + \frac{b(b-1)}{2}\frac{1}{n^2} + \dots)
\end{align}
Subtracting $n^b$ from $(n+1)^b$:
\begin{align}
    (n+1)^b - n^b = n^b(1+b\frac{1}{n} + \frac{b(b-1)}{2}\frac{1}{n^2} + \dots) - n^b = \fO \qty(bn^{b-1}).
\end{align}
The leading order of the denominator of \eqref{eq: nb frac} is clearly $n^{2b}$, which gives 
\begin{align}
    \alpha_n - \alpha_{n+1} = \fO \qty(\frac{bn^{b-1}}{n^{2b}}) = \fO \qty( \frac{b}{n^{b+1}}).
\end{align}
Therefore with $\tau_n = \fO \qty(n^{1-b})$,
\begin{align}
\sum_{k=1}^n \abs{\alpha_k - \alpha_{k+1}}\tau_k = \fO \qty( b\sum_{k=1}^n \frac{1}{k^{2b}})
\end{align}
which clearly converges as $n\rightarrow \infty$ for $\frac{4}{5} < b < 1$.
\end{proof}

\paragraph{\eqref{eq: ai2 sum aj tj}:}
\
\begin{proof}
Case 1: In the proof for \eqref{eq: tau_n order} we prove that $\sum_{k=1}^n \alpha_k = \fO \qty(\log n)$ when $b=1$. Then since $\tau_k$ is increasing, we have 
\begin{equation}
    \sum_{k=1}^n \alpha_k^2 \sum_{j=1}^{k-1}\alpha_j\tau_j \leq  \sum_{k=1}^n \alpha_k^2 \tau_k \sum_{j=1}^{k-1}\alpha_j = \fO \qty(\sum_{k=1}^n \frac{\log^2 k}{k^2}),
\end{equation}
which clearly converges as $n \rightarrow \infty$.

Case 2: For the case when $b\in (\frac{4}{5},1)$, we first consider the inner sum of \eqref{eq: ai2 sum aj tj},
\begin{align}
    \sum_{j=1}^{k-1}\alpha_j\tau_j = \fO \qty( \sum_{j=1}^{k-1}\frac{1}{j^{2b-1}}),
\end{align}
which we can approximate by an integral,
\begin{align}
    \int_1^k \frac{1}{x^{2b-1}} \ dx = \fO \qty( k^{2-2b}).
\end{align}
Therefore,
\begin{align}
    \sum_{k=1}^n \alpha_k^2 \sum_{j=1}^{k-1}\alpha_j\tau_j = \fO \qty( \sum_{k=1}^n \frac{k^{2-2b}}{k^{2b}}) = \fO \qty( \sum_{k=1}^n \frac{1}{k^{4b-2}}),
\end{align}
which converges for $\frac{4}{5} < b \leq 1$ as $n \rightarrow \infty$.
\end{proof}

\paragraph{\eqref{eq: ai aj^2 tj^2}:}
\begin{proof}

Case 1: For $b=1$, because we have $\alpha_{j,i} < \alpha_{j+1,i}$ and $\alpha_{i,i} = \alpha_i$ from Lemma \ref{lem:bravo b1}, we have the order-wise approximation,
\begin{align}
    \sum_{i=1}^n \alpha_i \sqrt{\sum_{j=1}^{i-1}\alpha_{j,i-1}^2\tau_{j-1}^2} 
        &\leq \sum_{i=1}^n \alpha_i \sqrt{ \alpha_{i-1}^2 \tau_{i-1}^2 \sum_{j=1}^{i-1} 1}, \explain{$\tau_i$ is increasing}\\
        &=\sum_{i=1}^n \alpha_i \alpha_{i-1} \tau_{i-1} \sqrt{i-1}. \\
        &= \fO \qty( \sum_{i=1}^n\frac{\log(i-1)}{i\sqrt{(i-1)}}) \\
        &= \fO \qty( \sum_{i=1}^n \frac{\log(i-1)}{i^{3/2}}),
\end{align}
which clearly converges. 

Case 2: For the case when $b\in (\frac{4}{5},1)$, we have,
\begin{align}
    \sum_{i=1}^n \alpha_i \sqrt{\sum_{j=1}^{i-1}\alpha_{j,i-1}^2\tau_{j-1}^2} 
        &\leq \sum_{i=1}^n \alpha_i \tau_{i-1} \sqrt{ \sum_{j=1}^{i-1} \alpha_{j,i-1}^2}, \explain{$\tau_i$ is increasing} \\
        &=\sum_{i=1}^n \alpha_i \tau_{i-1} \sqrt{\alpha_{i}}. \explain{Lemma \ref{lem:bravo_b.2}} \\
        &= \fO \qty( \sum_{i=1}^n\frac{i^{1-b}}{i^b \sqrt{i^b}}) \\
        &= \fO \qty(\sum_{i=1}^n \frac{1}{i^{5b/2 -1}}),
\end{align}
which converges for $\frac{4}{5}< b<1$. 
\end{proof}

Then, under Assumption \ref{as:e1}, we prove additional results about the convergence of the first and second moments of the additive noise $\qty{\e{1}_n}$.
\begin{lemma}\label{lem: e1 lr}
Let Assumptions~\ref{as:lr} and~\ref{as:e1} hold. Then, we have
    \begin{align}
        \E\qty[\norm{\e{1}_n}]&= \fO\qty(\frac{1}{\sqrt{n}}), \label{eq: first moment}\\
        \sup_n \sum_{k=1}^n \alpha_{k}\E\qty[\norm{\e{1}_{k}}] &< \infty, \label{eq: e1 lr ake1}\\
        \sup_n \sum_{k=1}^n \alpha_k \E\qty[\norm{\e{1}_k}^2] &< \infty, \label{eq: e1 lr akek2}\\
        \sup_n \sum_{k=1}^n \alpha_k^2 \E\qty[\norm{\e{1}_k}^2] &< \infty, \label{eq: e1 lr ak2ek2}\\
        \sup_n \sum_{k=1}^n \alpha_k \sum_{j=1}^{k-1} \alpha_{j,k-1} \E\qty[\norm{\e{1}_j}] &< \infty \label{eq: e1 lr ak sum ake1}.
    \end{align}
\end{lemma}
\begin{proof}
Recall that by Assumption \ref{as:e1} we have $\E\qty[\norm{\e{1}_n}^2] = \fO\qty(\frac{1}{n})$. 
Also recall that $\alpha_k = \fO \qty( \frac{1}{n^b})$ with $\frac{4}{5} < b \leq 1$. Then, we can prove the following equations:
\paragraph{\eqref{eq: first moment}:}
 By Jensen's inequality, we have
\begin{align}
    \E\qty[\norm{\e{1}_n}]\leq \sqrt{\E\qty[\norm{\e{1}_n}^2]} = \fO\qty(\frac{1}{\sqrt{n}}).
\end{align}

\paragraph{\eqref{eq: e1 lr ake1}:}
\begin{equation}
    \sum_{k=1}^n \alpha_{k}\E\qty[\norm{\e{1}_{k}}] = \fO \qty( \sum_{k=1}^n \frac{1}{k^{b+\frac{1}{2}}})
\end{equation}
which clearly converges as $n \rightarrow \infty$ for $\frac{4}{5}<b\leq 1$. 

\paragraph{\eqref{eq: e1 lr akek2}:}
\begin{equation}
    \sum_{k=1}^n \alpha_{k}\E\qty[\norm{\e{1}_{k}}^2] = \fO \qty( \sum_{k=1}^n \frac{1}{k^{b+1}}) \label{eq: ak expect conv}
\end{equation}
which clearly converges as $n \rightarrow \infty$ for $\frac{4}{5}<b\leq 1$. 

\paragraph{\eqref{eq: e1 lr ak2ek2}:}
\begin{equation}
    \sum_{k=1}^n \alpha_{k}^2\E\qty[\norm{\e{1}_{k}}^2] = \fO \qty( \sum_{k=1}^n \frac{1}{k^{2b+1}})
\end{equation}
which clearly converges as $n \rightarrow \infty$ for $\frac{4}{5}<b\leq 1$. 

\paragraph{\eqref{eq: e1 lr ak sum ake1}:}
\begin{align}
    \sum_{k=1}^n \alpha_k \sum_{j=1}^{k-1} \alpha_{j,k-1} \E\qty[\norm{\e{1}_j}] &\leq \sum_{k=1}^n \alpha_k^2 \sum_{j=1}^{k-1}\E\qty[\norm{\e{1}_j}], \explain{Lemma \ref{lem:bravo b1}} \\
    &= \fO \qty( \sum_{k=1}^n \frac{1}{k^{2b}} \sum_{j=1}^{k-1}\frac{1}{\sqrt{j}}). \explain{Lemma \ref{lem: e1 lr}}
\end{align}
It can be easily verified with an integral approximation that $\sum_{j=1}^{k-1} \frac{1}{\sqrt{j}} = \fO( \sqrt{k})$. This further implies
\begin{align}
     \sum_{k=1}^n \alpha_k \sum_{j=1}^{k-1} \alpha_{j,k-1} \E\qty[\norm{\e{1}_j}] &= \fO \qty(\sum_{k=1}^n \frac{1}{k^{2b - \frac{1}{2}}}), 
\end{align}
which converges as $n \rightarrow \infty$ for $\frac{4}{5} < b \leq 1$.
\end{proof}

Next, in Lemma \ref{lem:xn_norm}, we upper-bound the iterates $\qty{x_n}$.
\begin{lemma}\label{lem:xn_norm}
For each $\qty{x_n}$, we have 
\begin{equation}
    \|x_n\| \leq \norm{x_0} + C_H\sum_{k=1}^n \alpha_{k} +\sum_{k=1}^n \alpha_{k}\norm{\e{1}_{k}} \leq C_{\ref{lem:xn_norm}}\tau_n +\sum_{k=1}^n \alpha_{k}\norm{\e{1}_{k}},
\end{equation}
where $C_{\ref{lem:xn_norm}}$ is a deterministic constant.
\end{lemma}

\begin{proof}
Applying $\|\cdot\|$ to both sides of \eqref{eq:skm_markov} gives,
\begin{align}
\|x_{n+1}\| &= \norm{(1-\alpha_{n+1})x_n + \alpha_{n+1}\qty(H\paren{x_n,Y_{n+1}}+\e{1}_{n+1})}, \\
&\leq (1-\alpha_{n+1})\|x_n\| + \alpha_{n+1}\norm{H(x_n,Y_{n+1})} + \alpha_{n+1}\norm{\e{1}_{n+1}}, \\
&\leq (1-\alpha_{n+1})\norm{x_n} + \alpha_{n+1}\paren{C_H+\norm{x_n}} + \alpha_{n+1}\norm{\e{1}_{n+1}}, \explain{By \eqref{eq:H linear growth}}\\
&= \|x_n\|+\alpha_{n+1}C_H + \alpha_{n+1}\norm{\e{1}_{n+1}}.
\end{align}
A simple induction shows that almost surely,
\begin{align}
\norm{x_n} \leq \norm{x_0} + C_H\sum_{k=1}^n \alpha_{k} +\sum_{k=1}^n \alpha_{k}\norm{\e{1}_{k}}.
\end{align}
Since $\qty{\alpha_n}$ is monotonically decreasing, we have 
\begin{align}
\norm{x_n} &\leq \norm{x_0} + \frac{C_H}{\paren{1-\alpha_1}} \sum_{k=1}^n \alpha_k(1-\alpha_k) +\sum_{k=1}^n \alpha_{k}\norm{\e{1}_{k}},\\
&= \norm{x_0} + \frac{C_H}{\paren{1-\alpha_1}}\tau_n +\sum_{k=1}^n \alpha_{k}\norm{\e{1}_{k}}, \\
&\leq \max\left\{\norm{x_0},\frac{C_H}{\paren{1-\alpha_1}}\right\}\paren{1+ \tau_n} +\sum_{k=1}^n \alpha_{k}\norm{\e{1}_{k}}.
\end{align}
Therefore, since $\tau_n$ is monotonically increasing, there exists some constant we denote as $C_{\ref{lem:xn_norm}}$ such that
\begin{align}
    \norm{x_n} &\leq C_{\ref{lem:xn_norm}}\tau_n +\sum_{k=1}^n \alpha_{k}\norm{\e{1}_{k}}.
\end{align}
\end{proof}

\begin{lemma}\label{lem:v Lipschitz}
With $\nu(x,y)$ as defined in \eqref{eq: Poisson decomp}, we have
\begin{equation}
    \label{eq nu lipschitz}
    \norm{\nu\paren{x,y} - \nu\paren{x',y}} \leq C_{\ref{lem:v Lipschitz}}\norm{x - x'},
\end{equation}
which further implies
\begin{equation}
    \norm{\nu\paren{x,y}} \leq C_{\ref{lem:v Lipschitz}}\qty(C_{\ref{lem:v Lipschitz}}'+\norm{x}),
\end{equation}
where $C_{\ref{lem:v Lipschitz}},C_{\ref{lem:v Lipschitz}}'$ are deterministic constants.
\end{lemma}
\begin{proof}
Since we work with a finite $\fY$, we will use functions and matrices interchangeably. For example, given a function $f: \fY \to \R^d$, we also use $f$ to denote a matrix in $\R^\qty(\ny \times d)$ whose $y$-th row is $f(y)^\top$. Similarly, a matrix in $\R^\qty(\ny \times d)$ also corresponds to a function $\fY \to \R^d$. 

Let $\nu_x \in \R^{\ny \times d}$ denote the function $y \mapsto \nu(x, y)$
and let $H_x \in \R^{\ny \times d}$ denote the function $y \mapsto H(x, y)$.
Theorem 8.2.6 of \citet{puterman2014markov}
then ensures that
\begin{align}
\label{eq: v Hy Hx}
    \nu_x = H_\fY H_x,
\end{align}
where $H_\fY \in \R^{\ny \times \ny}$ is the fundamental matrix of the Markov chain depending only on the chain's transition matrix $P$.
The exact expression of $H_\fY$ is inconsequential and we refer the reader to \citet{puterman2014markov} for details.
Then we have for any $i=1,\dots, d$,
\begin{align}
    \nu_x[y, i] = \sum_{y'} H_\fY[y, y'] H_x[y', i].
\end{align}
This implies that
\begin{align}
    \abs{\nu_x[y, i] - \nu_{x'}[y, i]} \leq& \sum_{y'} H_\fY[y, y'] \abs{H_x[y', i] - H_{x'}[y', i]} \\
    \leq& \sum_{y'} H_\fY[y, y'] \norm{H(x, y) - H(x', y')}_\infty \\
    \leq& \sum_{y'} H_\fY[y, y'] \norm{x - x'}_\infty \explain{Assumption \ref{as:H 1 Lipschitz}}\\
    \leq& \norm{H_\fY}_\infty \norm{x - x'}_\infty,
\end{align}
yielding
\begin{align}
    \norm{\nu(x, y) - \nu(x', y)}_\infty \leq \norm{H_\fY}_\infty \norm{x - x'}_\infty.
\end{align}
The equivalence between norms in finite dimensional space ensures that there exists some $C_{\ref{lem:v Lipschitz}}$ such that~\eqref{eq nu lipschitz} holds.
Letting $x'=0$ then yields 
\begin{align}
    \norm{\nu(x, y)} \leq C_{\ref{lem:v Lipschitz}} \qty(\norm{\nu(0, y)} + \norm{x}).
\end{align}
Define $C_{\ref{lem:v Lipschitz}}' \doteq \max_y \norm{\nu(0, y)}$,
we get
\begin{align}
    \norm{\nu(x, y)} \leq C_{\ref{lem:v Lipschitz}}\qty(C_{\ref{lem:v Lipschitz}}' + \norm{x}). \label{eq:nu linear growth}
\end{align}

\end{proof}

\begin{lemma}\label{lem:v_norm}
We have for any $y \in \fY$,
\begin{equation}
    \norm{\nu\paren{x_n,y}} \leq 
    \zeta_{\ref{lem:v_norm}} \tau_n,
\end{equation}
where $\zeta$ is a possibly sample-path dependent constant. Additionally, we have
\begin{align}
    \E\qty[\norm{\nu \paren{x_n,y}}] &\leq C_{\ref{lem:v_norm}} \tau_n,
\end{align}
where $ C_{\ref{lem:v_norm}} $ is a deterministic constant.
\end{lemma}
\begin{proof}
    Having proven that $\nu\paren{x,y}$ is Lipschitz continuous in $x$ in Lemma \ref{lem:v Lipschitz}, we have
    \begin{align}
        \norm{\nu\paren{x_{n},y}} 
        &\leq C_{\ref{lem:v Lipschitz}}(C_{\ref{lem:v Lipschitz}}' + \norm{x_n}), \explain{Lemma \ref{lem:v Lipschitz}}\\
        &\leq C_{\ref{lem:v Lipschitz}}\qty(C_{\ref{lem:v Lipschitz}}' + C_{\ref{lem:xn_norm}}\tau_n +\sum_{k=1}^n \alpha_{k}\norm{\e{1}_{k}}). \explain{Lemma \ref{lem:xn_norm}}\\
        &= \fO\qty(\tau_n + \sum_{k=1}^n \alpha_{k}\norm{\e{1}_{k}}). \label{eq: v intermediate}
    \end{align}
Since \eqref{as: total noise} in Assumption \ref{as:e1} assures us that $\sum_{k=1}^\infty \alpha_{k}\norm{\e{1}_{k}}$ is finite almost surely while $\tau_n$ is monotonically increasing, then there exists some possibly sample-path dependent constant $\zeta_{\ref{lem:v_norm}}$ such that 
\begin{equation}
    \norm{\nu\paren{x_{n},y}} \leq \zeta_{\ref{lem:v_norm}} \tau_n.
\end{equation}

We can also prove a deterministic bound on the expectation of $\norm{\nu\paren{x_{n},Y_{n+1}}}$,
\begin{align}
    \E\qty[\norm{\nu\paren{x_{n},y}}] &= \fO \qty(\E\qty[\tau_n+ \sum_{k=1}^n \alpha_{k}\norm{\e{1}_{k}}]), \\
    &= \fO \qty(\tau_n+ \sum_{k=1}^n \alpha_{k}\E\qty[\norm{\e{1}_{k}}]).
\end{align}
By Lemma \ref{lem: e1 lr}, its easy to see that $\sum_{k=1}^n \alpha_{k}\E\qty[\norm{\e{1}_{k}}]< \infty$.
Therefore, there exists some deterministic constant $C_{\ref{lem:v_norm}}$ such that
\begin{equation}
    \E\qty[\norm{\nu\paren{x_{n},y}}] \leq C_{\ref{lem:v_norm}}\tau_n.
\end{equation}
\end{proof}
Although the two statements in Lemma \ref{lem:v_norm} appear similar, their difference is crucial. Assumption \ref{as:e1} and \eqref{as: total noise} only ensure the existence of a sample-path dependent constant $\zeta_{\ref{lem:v_norm}}$ but its form is unknown, preventing its use for expectations or explicit bounds. In contrast, using \eqref{as: second moment} from Assumption \ref{as:e1}, we derive a universal constant $C_{\ref{lem:v_norm}}$.

\begin{lemma}\label{lem:M_norm_bound}For each $\qty{M_{n}}$, defined in \eqref{eq:M_define}, we have
\begin{equation}
\norm{M_{n+1}}\leq \zeta_{\ref{lem:M_norm_bound}} \tau_n,
\end{equation}
where $\zeta_{\ref{lem:M_norm_bound}}$ is a the sample-path dependent constant. 
\end{lemma}
\begin{proof}
Applying $\norm{\cdot}$ to \eqref{eq:M_define} gives
\begin{align}
    \norm{M_{n+1}} &= \norm{\nu\paren{x_n,Y_{n+2}}-P\nu\paren{x_n,Y_{n+1}} }, \\
    &\leq \norm{P\nu\paren{x_n,Y_{n+1}}} + \norm{\nu\paren{x_n,Y_{n+2}}},\\
    &= \norm{\sum_{y'\in \fY}P(Y_{n+1},y')\nu(x_n,y')}+\norm{\nu\paren{x_n,Y_{n+2}}}, \\
    &\leq \sum_{y'\in \fY}\norm{P(Y_{n+1},y')\nu(x_n,y')}+\norm{\nu\paren{x_n,Y_{n+2}}},\\
    &= \qty(\max_{y \in \fY} \norm{\nu(x_n,y)}) \sum_{y'\in \fY}\abs{P(Y_{n+1},y')}+\norm{\nu (x_n,Y_{n+2})},\\
    &\leq 2 \max_{y \in \fY} \norm{\nu(x_n,y)} \label{eq:M intermediate}
\end{align}
Under Assumption \ref{as:e1}, we can apply the sample-path dependent bound from Lemma \ref{lem:v_norm}, 
\begin{align}
    \norm{M_{n+1}} &\leq 2\zeta_{\ref{lem:v_norm}}\tau_n, \explain{Lemma \ref{lem:v_norm}}\\    
    &= \zeta_{\ref{lem:M_norm_bound}}\tau_n,
\end{align}
with $\zeta_{\ref{lem:M_norm_bound}} \doteq 2\zeta_{\ref{lem:v_norm}}$.
\end{proof}

\begin{lemma} \label{lem: M second moment}
For each $\qty{M_{n}}$, defined in \eqref{eq:M_define}, we have 
\begin{equation}
    \E\qty[\norm{M_{n+1}}^2 \mid \fF_{n+1}] \leq C_{\ref{lem: M second moment}}' (1+ \norm{x_{n}}^2), \label{eq: Mn squared growth}
\end{equation}
and
\begin{equation}
\E\qty[\norm{M_{n+1}}_2^2] \leq C_{\ref{lem: M second moment}}^2 \tau_n^2, \label{eq: M second moment bound}
\end{equation}
where $C_{\ref{lem: M second moment}}'$ and $C_{\ref{lem: M second moment}}$ are deterministic constants and 
\begin{align}
    \fF_{n+1} \doteq \sigma(x_0, Y_1, \dots, Y_{n+1})
\end{align}
is the $\sigma$-algebra until time $n+1$.
\end{lemma}
\begin{proof} 
First, to prove \eqref{eq: Mn squared growth},
we have 
\begin{align}
    \E \qty[\norm{M_{n+1}}^2 \mid \fF_{n+1}] &\leq 4 \max_{y \in \fY} \norm{\nu(x_n,y)}^2 = \fO\qty(1+ \norm{x_n}^2),
\end{align}
where the first inequality results form \eqref{eq:M intermediate} in Lemma \ref{lem:M_norm_bound} and the second inequality results from Lemma \ref{lem:v Lipschitz}.

    Then, to prove \eqref{eq: M second moment bound}, from Lemma \ref{lem:xn_norm} we then have,
    \begin{align}
        \E \qty[\norm{\nu\qty(x_{n},y)}^2] &\leq \E \qty[1+\qty(C_{\ref{lem:xn_norm}}\tau_n+ \sum_{k=1}^n \alpha_{k}\norm{\e{1}_{k}})^2]
         = \fO \qty(\tau_n^2+ \E\qty[\qty(\sum_{k=1}^n \alpha_{k}\norm{\e{1}_{k}})^2]).
    \end{align}
 Recall that by Assumption \ref{as:e1}, $\E\qty[\norm{\e{1}_k}^2] = \fO \qty(\frac{1}{k})$. Examining the right-most term we then have,
 \begin{align}
    \E\qty[\qty(\sum_{k=1}^n \alpha_{k}\norm{\e{1}_{k}})^2] &\leq \E\qty[\qty(\sum_{k=1}^n \alpha_k)\qty(\sum_{k=1}^n \alpha_k \norm{\e{1}_k}^2)], \explain{Cauchy-Schwarz}\\
    &= \fO\qty(\sum_{k=1}^n \alpha_k), \explain{By \eqref{eq: e1 lr akek2} in Lemma \ref{lem: e1 lr}}\\
    &= \fO\qty(\frac{1}{1 - \alpha_1}\sum_{k=1}^n \alpha_k (1 - \alpha_1)),\\ 
    &= \fO\qty(\sum_{k=1}^n \alpha_k (1 - \alpha_k)), \\
    &= \fO \qty(\tau_n).
\end{align}
We then have
\begin{align}
    \E \qty[\norm{\nu \qty(x_n,y)}^2] = \fO(\tau_n^2).\label{eq: v second moment}
\end{align}
Because our bound on $\E \qty[\norm{\nu \qty(x_n,y)}^2]$ is independent of $y$, we have
\begin{align}
    \E \qty[\norm{M_{n+1}}^2] &= \fO \qty(\E \qty[\norm{\nu(x_n,y)}^2]) = \fO(\tau_n^2). \explain{By \eqref{eq: v second moment}}
\end{align}
Due to the equivalence of norms in finite-dimensional spaces, there exists a deterministic constant $C_{\ref{lem: M second moment}}$ such that \eqref{eq: M second moment bound} holds.
\end{proof}

Now, we are ready to present four additional lemmas which we will use to bound the four noise terms in \eqref{eq:sumUn_derivation}.
\begin{lemma}\label{lem:sup_M}
With $\qty{\Mbb_n}$ defined in \eqref{eq:sumUn_derivation},
\begin{equation}
    \lim_{n \rightarrow \infty} \Mbb_n < \infty, \qq{a.s.} 
\end{equation}
\end{lemma}
\begin{proof}
We first observe that the sequence $\qty{\Mbb_n}$ defined in \eqref{eq:sumUn_derivation} is positive and monotonically increasing. Therefore by the monotone convergence theorem, it converges almost surely to a (possibly infinite) limit which we denote as,
\begin{equation}
    \Mbb_\infty \doteq \lim_{n \rightarrow \infty }\Mbb_n \qq{a.s.}
\end{equation}
Then, we will utilize a generalization of Lebesgue's monotone convergence theorem (Lemma \ref{lem: beppo levi}) to prove that the limit $\Mbb_\infty$ is finite almost surely. From Lemma \ref{lem: beppo levi}, we see that 
\begin{align}
    \E\qty[\Mbb_\infty] = \lim_{n\rightarrow \infty} \E \qty[\Mbb_n].
\end{align}
Therefore, to prove that $\Mbb_\infty$ is almost surely finite, it is sufficient to prove that $\lim_{n\rightarrow \infty} \E \qty[\Mbb_n] < \infty$. To this end, we proceed by bounding the expectation of $\qty{\Mbb_{n}}$, by first starting with $\qty{\Mb_n}$ from \eqref{eq:Un_norm_bar_bounds}. We have,
\begin{align}
    \E\left[\norm{ \Mb_n} \right] &= \E\left[\norm{\sum_{i=1}^{n} \alpha_{i,n} M_i}\right], \\
    &= \fO\qty( \sqrt{\E\left[\norm{\sum_{i=1}^{n} \alpha_{i,n} M_i}_2^2\right]}), \explain{Jensen's Ineq.} \\
    &=\fO\qty(\sqrt{\sum_{i=1}^{n} \alpha_{i,n}^2 \E\left[\norm{M_i}_2^2\right]}), \explain{$M_i$ is a Martingale Difference Series}\\
    &= \fO\qty(\sqrt{\sum_{i=1}^{n} \alpha_{i,n}^2 \tau_{i}^2}), \explain{Lemma \ref{lem: M second moment}} \label{eq: M stop}
\end{align}
Then using the definition of $\qty{\Mbb_n}$ from \eqref{eq:sumUn_derivation}, we have
\begin{align}
    \E\qty[\Mbb_{n}]  &= \sum_{i=1}^n \alpha_i \E\left[\norm{\Mb_{i-1}}\right] = \fO\qty( \sum_{i=1}^n \alpha_i \sqrt{\sum_{j=1}^{i-1} \alpha_{j,i-1}^2 \tau_{j-1}^2}).
\end{align}
Then, by \eqref{eq: ai aj^2 tj^2} in Lemma \ref{lem:lr}, we have
\begin{equation}
    \sup_n \E\qty[\Mbb_{n}] < \infty,
\end{equation}
and since $\qty{\E\qty[\Mbb_{n}]}$ is also monotonically increasing, we have
\begin{equation}
    \lim_{n \rightarrow \infty} \E\qty[\Mbb_{n}] < \infty,
\end{equation}
which implies that $\Mbb_\infty < \infty$ almost surely.
\end{proof}

\begin{lemma}\label{lem:sup_e1}
With $\qty{\ebb{1}_n}$ defined in \eqref{eq:sumUn_derivation},
\begin{equation}
    \lim_{n \rightarrow \infty} \ebb{1}_n < \infty, \  \qq{a.s.}
\end{equation}
\end{lemma}

\begin{proof}
We first observe that the sequence $\qty{\ebb{1}_n}$ defined in \eqref{eq:sumUn_derivation} is positive and monotonically increasing. Therefore by the monotone convergence theorem, it converges almost surely to a (possibly infinite) limit which we denote as,
\begin{equation}
    \ebb{1}_\infty \doteq \lim_{n \rightarrow \infty }\ebb{1}_n \qq{a.s.}
\end{equation}

Then, we utilize a generalization of Lebesgue's monotone convergence theorem (Lemma \ref{lem: beppo levi}) to prove that the limit $\ebb{1}_\infty$ is finite almost surely. By Lemma \ref{lem: beppo levi}, we have
\begin{align}
    \E\qty[\ebb{1}_\infty] = \lim_{n\rightarrow \infty} \E \qty[\ebb{1}_n].
\end{align}
Therefore, to prove that $\ebb{1}_\infty$ is almost surely finite, it is sufficient to prove that $\lim_{n\rightarrow \infty} \E \qty[\ebb{1}_n] < \infty$. To this end, we proceed by bounding the expectation of $\qty{\ebb{1}_{n}}$,
\begin{align}
    \E\left[\ebb{1}_n\right] = \sum_{i=1}^n \alpha_i \E\left[\norm{\eb{1}_{i-1}}\right]
    \leq \sum_{i=1}^n \alpha_i \sum_{j=1}^{i-1} \alpha_{j,i-1} \E\qty[\norm{\e{1}_j}].
\end{align}
Then, by \eqref{eq: e1 lr ak sum ake1} in Lemma \ref{lem: e1 lr}, we have,
\begin{equation}
    \sup_n \E\qty[\ebb{1}_{n}] < \infty,
\end{equation}
and since $\qty{\E\qty[\ebb{1}_{n}]}$ is also monotonically increasing, we have
\begin{equation}
    \lim_{n \rightarrow \infty} \E\qty[\ebb{1}_{n}] < \infty.
\end{equation}
which implies that $\ebb{1}_\infty < \infty$ almost surely.

\end{proof}

\begin{lemma}\label{lem:sup_e3}
With $\qty{\ebb{3}_n}$ defined in \eqref{eq:sumUn_derivation}, we have
\begin{equation}
    \lim_{n \rightarrow \infty} \  \ebb{3}_n < \infty,  \qq{a.s.}
\end{equation}
\end{lemma}
\begin{proof}
Beginning with the definition of $\eb{3}_n$ in \eqref{eq:Un_norm_bar_bounds}, we have
\begin{align}
    \norm{\eb{3}_n} &= \norm{\sum_{i=1}^n \alpha_{i,n}\paren{ \nu\paren{x_{i},Y_{i+1}} - \nu\paren{x_{i-1},Y_{i+1}}}}, \\
    &\leq \sum_{i=1}^n \alpha_{i,n}\norm{ \nu\paren{x_{i},Y_{i+1}} - \nu\paren{x_{i-1},Y_{i+1}}}, \\
    &\leq C_{\ref{lem:v Lipschitz}}\sum_{i=1}^n \alpha_{i,n}\norm{ x_i - x_{i-1}}, \explain{Lemma \ref{lem:v Lipschitz}}\\
    &\leq C_{\ref{lem:v Lipschitz}} \sum_{i=1}^n \alpha_{i,n}\alpha_{i}\qty(\norm{H\paren{x_{i-1},Y_{i}}}+\norm{x_{i-1}} + \norm{\e{1}_i}), \explain{By \eqref{eq:skm_markov}}\\
    &\leq C_{\ref{lem:v Lipschitz}} \sum_{i=1}^n \alpha_{i,n}\alpha_{i}\qty(2\norm{x_{i-1}}+C_H+ \norm{\e{1}_i}),
    \explain{By \eqref{eq:H linear growth}}\\
    &\leq C_{\ref{lem:v Lipschitz}} \sum_{i=1}^n \alpha_{i,n}\alpha_{i}\qty(2C_{\ref{lem:xn_norm}}\tau_{i-1} +2\sum_{k=1}^{i-1} \alpha_{k}\norm{\e{1}_{k}} +C_H + \norm{\e{1}_i}), \explain{Lemma \ref{lem:xn_norm}} \label{eq: e3 stop}
\end{align}
Because Assumption \ref{as:e1} assures us that $\sum_{k=1}^\infty \alpha_{k}\norm{\e{1}_{k}}$ is almost surely finite, then there exists some sample-path dependent constant we denote as $\zeta_{\ref{lem:sup_e3}}$ where,
\begin{align}
    \norm{\eb{3}_n} &\leq \zeta_{\ref{lem:sup_e3}} \sum_{i=1}^n \alpha_{i,n}\alpha_{i} \qty(\tau_{i-1} + \norm{\e{1}_i}), \explain{Assumption \ref{as:e1}}\\
    &\leq \zeta_{\ref{lem:sup_e3}} \qty(\sum_{i=1}^n \alpha_{i,n}\alpha_{i}\tau_i + \sum_{i=1}^n \alpha_{i,n}\alpha_{i}\norm{\e{1}_i}), \explain{$\tau_i$ is increasing}\\
    &\leq \zeta_{\ref{lem:sup_e3}} \alpha_n \qty(\sum_{i=1}^n \alpha_i \tau_i + \sum_{i=1}^n \alpha_{i}\norm{\e{1}_i}). \explain{Lemma \ref{lem:bravo b1}}.
\end{align}
Again, from Assumption \ref{as:e1} we can conclude that there exists some other sample-path dependent constant we denote as $\zeta_{\ref{lem:sup_e3}}'$ where
\begin{equation}
    \norm{\eb{3}_n} \leq \zeta_{\ref{lem:sup_e3}}' \alpha_n \sum_{i=1}^n \alpha_i \tau_i.
\end{equation}
Therefore, from the definition of $\ebb{3}_n$ in \eqref{eq:sum_Uk_converges}
\begin{equation}
    \ebb{3}_n \leq \zeta_{\ref{lem:sup_e3}}' \sum_{i=1}^n \alpha_i^2 \sum_{j=1}^{i-1} \alpha_j \tau_j.
\end{equation}
So, by \eqref{eq: ai2 sum aj tj} in Lemma \ref{lem:lr}
\begin{equation}
    \sup_n \  \ebb{3}_n \leq \sup_n \ \zeta_{\ref{lem:sup_e3}}' \sum_{i=1}^n \alpha_i^2 \sum_{j=1}^{i-1} \alpha_j \tau_j < \infty \qq{a.s.}
\end{equation}
Then, the monotone convergence theorem proves the lemma.
\end{proof}

To prove \eqref{eq:limit_to_0} holds almost surely, we introduce four lemmas which we will subsequently use to prove an extension of Theorem 2 from \citep{borkar2009stochastic} in Section \ref{sec:borkar_ext}. 
\begin{lemma} \label{lem:martingale_series_bound}
    We have 
    \begin{equation}
        \sup_{n} \norm{\sum_{k=1}^n\alpha_k M_k} < \infty \qq{a.s.}
    \end{equation}
\end{lemma}
\begin{proof}
 Recall that $M_k$ is a Martingale difference series. Then, the Martingale sequence $\qty{\sum_{k=1}^n \alpha_k M_k}$
    is bounded in $L^2$ with,
    \begin{align}
        \E\left[\norm{\sum_{k=1}^n \alpha_k M_k}_2\right] &\leq \sqrt{\E \left[ \norm{\sum_{k=1}^n \alpha_k M_k }_2^2\right]}, \explain{Jensen's Ineq.}\\
        &= \sqrt{\sum_{k=1}^n \alpha_k^2 \E \left[ \norm{M_k}_2^2 \right]}, \explain{$M_i$ is a Martingale Difference Series} \\
        &\leq C_{\ref{lem: M second moment}} \sqrt{\sum_{k=1}^n \alpha_k^2 \tau_k^2}. \explain{Lemma \ref{lem: M second moment}}
    \end{align}
    Lemma \ref{lem:lr} then gives 
    \begin{align}
         \sup_n \ C_{\ref{lem: M second moment}} \sqrt{\sum_{k=1}^n \alpha_k^2 \tau_k^2} &< \infty
    \end{align}
    
    Doob's martingale convergence theorem implies that $\qty{\sum_{k=1}^n \alpha_k M_k}$ converges to an almost surely finite random variable, which proves the lemma.
\end{proof}

\begin{lemma} \label{lem:ak_e2}
    We have,
    \begin{align}
        \sup_{n} \norm{\sum_{k=1}^n\alpha_k\e{2}_k} < \infty \qq{a.s.}
    \end{align}
\end{lemma}
\begin{proof}
Utilizing the definition of $\e{2}_k$ in \eqref{eq:e2_define}, we have
    \begin{align}
        \sum_{k=1}^n \alpha_{k}\e{2}_k &= -\sum_{k=1}^n \alpha_{k} \paren{\nu\paren{x_{k},Y_{k+1}} - \nu\paren{x_{k-1},Y_{k}}},\\
        &= -\sum_{k=1}^n \alpha_{k}\nu\paren{x_{k},Y_{k+1}} - \alpha_{k-1}\nu\paren{x_{k-1},Y_{k}}  + \alpha_{k-1}\nu\paren{x_{k-1},Y_{k}} - \alpha_{k} \nu\paren{x_{k-1},Y_{k}}, \\
        &= -\alpha_{n}\nu\paren{x_{n},Y_{n+1}} - \sum_{k=1}^n\paren{\alpha_{k-1}-\alpha_{k}}\nu\paren{x_{k-1},Y_k}. \explain{$\alpha_0$ = 0} \label{eq: ak e2 intermediate}
    \end{align}
    The triangle inequality gives
    \begin{align}
        \norm{\sum_{k=1}^n \alpha_{k}\e{2}_k}
        &\leq \alpha_n \norm{\nu\paren{x_{n},Y_{n+1}}} + \sum_{k=1}^n \left|\alpha_{k-1}-\alpha_{k} \right|\norm{\nu\paren{x_{k-1},Y_k}}, \\
        &\leq \zeta_{\ref{lem:v_norm}} \paren{\alpha_n\tau_n + \sum_{k=1}^n \left|\alpha_{k-1}-\alpha_{k} \right|\tau_{k-1}}, &&\text{(Lemma \ref{lem:v_norm})}\\
        &= \zeta_{\ref{lem:v_norm}} \qty(\alpha_n\tau_n  + \alpha_1 \tau_1 + \sum_{k=1}^{n-1} \abs{\alpha_{k}-\alpha_{k+1}}\tau_{k}) \explain{$\alpha_0 \doteq 0$}.
    \end{align}
Its easy to see that $\lim_{n \rightarrow \infty} \alpha_n\tau_n = 0$, and $\alpha_1 \tau_1$ is simply a deterministic and finite constant. Therefore, by Lemma \ref{lem:lr} we have
\begin{align}
    \sup_n \sum_{k=1}^{n} \abs{\alpha_{k}-\alpha_{k+1}}\tau_{k} &< \infty \qq{a.s.}
\end{align}
which proves the lemma.

\end{proof}

\begin{lemma} \label{lem:ak_e3}
    We have, 
    \begin{equation}
        \sup_{n}  \norm{\sum_{k=1}^n\alpha_k\e{3}_k} < \infty \qq{a.s.}
    \end{equation}
\end{lemma}
\begin{proof}
Utilizing the definition of $\e{3}_k$ in \eqref{eq:e3_define}, we have
\begin{align}
    \norm{\sum_{k=1}^n\alpha_k\e{3}_k} &= \norm{\sum_{k=1}^n \alpha_{k}\qty( \nu\qty(x_{k},Y_{k+1}) - \nu \qty(x_{k-1},Y_{k+1}))}, \\
    &\leq \sum_{k=1}^n \alpha_{k}\norm{ \nu\qty(x_{k},Y_{k+1}) - \nu\qty(x_{k-1},Y_{k+1})}, \\
    &\leq C_{\ref{lem:v Lipschitz}}\sum_{k=1}^n \alpha_{k}\norm{ x_k - x_{k-1}}, \explain{Lemma \ref{lem:v Lipschitz}}\\
    &\leq C_{\ref{lem:v Lipschitz}} \sum_{k=1}^n \alpha_{k}^2\qty(\norm{H\paren{x_{k-1},Y_{k}}} +\norm{x_{k-1}} + \norm{\e{1}_k}), \\
    & \quad \quad \quad \quad \text{(By \eqref{eq:skm_markov})} \\
    &\leq C_{\ref{lem:v Lipschitz}} \sum_{k=1}^n \alpha_{k}^2\qty(2\norm{x_{k-1}}+C_H+ \norm{\e{1}_k}),
    \explain{By \eqref{eq:H linear growth}}\\
    &\leq C_{\ref{lem:v Lipschitz}} \sum_{k=1}^n \alpha_{k}^2\qty(2C_{\ref{lem:xn_norm}}\tau_{k-1} +2\sum_{i=1}^{k-1} \alpha_{i}\norm{\e{1}_{i}} +C_H + \norm{\e{1}_k}). \explain{Lemma \ref{lem:xn_norm}}
\end{align}
Because Assumption \ref{as:e1} assures us that $\sum_{k=1}^\infty \alpha_{k}\norm{\e{1}_{k}}$ is finite, then there exists some sample-path dependent constant we denote as $\zeta_{\ref{lem:ak_e3}}$ where,
\begin{align}
    \norm{\sum_{k=1}^n\alpha_k\e{3}_k} &\leq \zeta_{\ref{lem:ak_e3}} \sum_{k=1}^n \alpha_{k}^2 \qty(\tau_{k-1} + \norm{\e{1}_k}), \explain{Assumption \ref{as:e1}}\\
    &\leq \zeta_{\ref{lem:ak_e3}} \qty(\sum_{k=1}^n \alpha_{k}^2\tau_k + \sum_{k=1}^n \alpha_{k}^2\norm{\e{1}_k}), \explain{$\tau_k$ is increasing}
\end{align}
Lemma $\ref{lem:lr}$ and Assumption \ref{as:e1} then prove the lemma.
\end{proof}

\begin{lemma}\label{lem:xi_noise_convergence}
    Let $U_n$ be the iterates defined in \eqref{eq:U_n_define}. Then if $\sup_n \norm{U_n}<\infty$, we have $U_n \rightarrow 0$ almost surely.
\end{lemma}
\begin{proof}
We use a stochastic approximation argument to show that $U_n \rightarrow 0$. The almost sure convergence of $U_n \rightarrow 0$ is given by a generalization of Theorem 2.1 of \cite{borkar2009stochastic}, which we present as Theorem \ref{thm:borkar_thm2} in Appendix \ref{sec:borkar_ext} for completeness. 

We now verify the assumptions of Theorem \ref{thm:borkar_thm2}. 
Beginning with the definition of $\xi_k$ in \eqref{eq:xi_define}, we have
\begin{align}
    \lim_{n\to\infty}\sup_{j\geq n} \norm{\sum_{k=n}^{j}\alpha_k \xi_k} &= \lim_{n\to\infty}\sup_{j\geq n} \norm{\sum_{k=n}^{j}\alpha_k \paren{\e{1}_k+\e{2}_k+ \e{3}_k}}, \\
    &\leq \underbrace{\lim_{n\to\infty}\sup_{j\geq n} \norm{\sum_{k=n}^{j}\alpha_k \e{1}_k}}_{S_1} + \underbrace{ \lim_{n\to\infty}\sup_{j\geq n} \norm{\sum_{k=n}^{j}\alpha_k \e{2}_k}}_{S_2}  + \underbrace{\lim_{n\to\infty}\sup_{j\geq n} \norm{\sum_{k=n}^{j}\alpha_k \e{3}_k}}_{S_3}. 
\end{align}
We now bound the three terms in the RHS.

For $S_1$, we have
\begin{align}
    \lim_{n\to\infty} \sup_{j\geq n} \norm{\sum_{k=n}^j \alpha_k \e{1}_k} \leq \lim_{n\to\infty} \sup_{j\geq n} \sum_{k=n}^j \alpha_k \norm{\e{1}_k} \leq \lim_{n\to\infty} \sum_{k=n}^\infty \alpha_k \norm{\e{1}_k} = 0,
\end{align}
where we have used the fact that the series $\sum_{k=1}^n \alpha_k \norm{\e{1}_k}$ converges by Assumption \ref{as:e1} almost surely.

For $S_2$, 
from \eqref{eq: ak e2 intermediate} in Lemma \ref{lem:ak_e2}, we have
\begin{align}
    \sum_{k=n}^j\alpha_k\e{2}_k &= \sum_{k=1}^j \alpha_k \e{2}_k - \sum_{k=1}^{n-1} \alpha_k \e{2}_k,\\
    & = \alpha_{n-1}\nu(x_n, Y_n)- \alpha_j \nu(x_j, Y_{j+1}) - \sum_{k=n}^j (\alpha_{k-1} - \alpha_k)\nu(x_{k-1}, Y_k).
\end{align}
Taking the norm and applying the triangle inequality, we have
\begin{align}
    \lim_{n\to\infty} \sup_{j\geq n} \norm{\sum_{k=n}^j\alpha_k\e{2}_k}  &\leq \lim_{n\to\infty} \sup_{j\geq n} \bigg(\alpha_{n-1} \norm{\nu(x_n, Y_n)}+ \alpha_j \norm{\nu(x_j, Y_{j+1})} \\ &\quad+ \sum_{k=n}^j \norm{(\alpha_{k-1} - \alpha_k)\nu(x_{k-1}, Y_k)}\bigg),\\
    &\leq \lim_{n\to\infty} \sup_{j\geq n} \zeta_{\ref{lem:v_norm}} \qty(\alpha_{n-1}\tau_{n-1}+ \alpha_j \tau_j + \sum_{k=n}^\infty \abs{\alpha_{k-1} - \alpha_k}\tau_{k-1}), \explain{Lemma \ref{lem:v_norm}}
\end{align}
where the last inequality holds because $\sum_{k=n}^j \abs{\alpha_{k-1} - \alpha_k}\tau_{k-1}$ is monotonically increasing.
Note that
\begin{align}
    \alpha_n\tau_n = \begin{cases} 
      \fO\qty(n^{1-2b}) & \text{if} \quad \frac{4}{5}<b<1, \\
      \fO\qty(\frac{\log n}{n}) & \text{if} \quad b=1.
      \end{cases} \label{eq: alpha_n tau_n order}
\end{align}
Since we have $j\geq n$, then
\begin{align}
    \lim_{n\to\infty} \sup_{j\geq n} \norm{\sum_{k=n}^j\alpha_k\e{2}_k}  &\leq \lim_{n\to\infty} \zeta_{\ref{lem:v_norm}} \qty(2\alpha_{n-1}\tau_{n-1}+ \sum_{k=n}^\infty \abs{\alpha_{k-1} - \alpha_k}\tau_{k-1}) = 0
\end{align}
where we used the fact that \eqref{eq: abs ai ti} in Lemma \ref{lem:lr} and the monotone convergence theorem prove that the series $\sum_{k=1}^n \abs{\alpha_{k} - \alpha_{k+1}}\tau_{k}$ converges almost surely. 

For $S_3$,
following the steps in Lemma \ref{lem:ak_e3} (which we omit to avoid repetition), we have,
\begin{align}
    \lim_{n\to\infty} \sup_{j\geq n} \norm{\sum_{k=n}^j\alpha_k\e{3}_k} &\leq \lim_{n\to\infty} \sup_{j\geq n} \zeta_{\ref{lem:ak_e3}} \qty(\sum_{k=n}^j \alpha_{k}^2\tau_k + \sum_{k=n}^j \alpha_{k}^2\norm{\e{1}_k}).
\end{align}
which further implies that 
\begin{align}
    \lim_{n\to\infty} \sup_{j\geq n} \norm{\sum_{k=n}^j\alpha_k\e{3}_k} &\leq \lim_{n\to\infty} 
 \zeta_{\ref{lem:ak_e3}} \qty(\sum_{k=n}^\infty \alpha_{k}^2\tau_k + \sum_{k=n}^\infty \alpha_{k}^2\norm{\e{1}_k}) = 0,
\end{align}
where we use the fact that, by \eqref{eq: ai2 ti} in Lemma \ref{lem:lr}, Assumption \ref{as:e1}, and the monotone convergence theorem, both series on the RHS series converge almost surely. 
Therefore we have proven that, 
\begin{equation}
    \lim_{n\rightarrow \infty} \sup_{j \geq n} \norm{\sum_{k=n}^{j}\alpha_k \xi_k} = 0 \qq{a.s.}
\end{equation}
thereby verifying Assumption~\ref{as:borkar_noise}.

Assumption \ref{as:borkar_lipschitz} is satisfied by \eqref{eq:h nonexpansive} which is the result of Assumption \ref{as:H 1 Lipschitz}. Assumption \ref{as:borkar_lr} is clearly met by the definition of the deterministic learning rates in Assumption \ref{as:lr}. 
Demonstrating Assumption \ref{as:borkar_mds} holds, Lemma \ref{lem: M second moment} demonstrates $\qty{M_n}$ is square-integrable martingale difference series.

Therefore, by Theorem \ref{thm:borkar_thm2}, the iterates $\qty{U_n}$ converge almost surely to a possibly sample-path dependent compact connected internally chain transitive set of the following ODE:
    \begin{align}
        \dv{U(t)}{t} = -U(t). \label{eq: Un limiting ODE}
    \end{align}
    Since the origin is the unique globally asymptotically stable equilibrium point of \eqref{eq: Un limiting ODE}, we have that $U_n \rightarrow 0$ almost surely. 
\end{proof}

\begin{lemma}\label{lem: applying bravo 2.1}
    With $\qty{x_n}$ defined in \eqref{eq:xi_define} and $\qty{U_n}$ defined in \eqref{eq:U_n_define}, if $\sum_{k=1}^\infty \alpha_k \norm{U_{k-1}}$ and $\lim_{n\rightarrow \infty}U_n = 0$, then $\lim_{n\rightarrow \infty} x_n = x_*$ where $x_* \in \fX_*$ is a possibly sample-path dependent fixed point.
\end{lemma}
\begin{proof}
    Following the approach of \citet{bravo2024stochastic}, we utilize the estimate for inexact Krasnoselskii-Mann iterations of the form \eqref{eq inexact km update} presented in Lemma \ref{lem:bravo_2.1} to prove the convergence of \eqref{eq:skm_markov}. Using the definition of $\left\{U_n\right\}$ in \eqref{eq:U_n_define}, we then let $z_0 = x_0$ and define $z_n \doteq x_n - U_n$, which gives
\begin{align}
z_{n+1} &= \paren{1-\alpha_{n+1}}x_{n} + \alpha_{n+1}\paren{h(x_{n})+M_{n+1}+\xi_{n+1}} \\
&\quad \quad - \paren{\paren{1-\alpha_{n+1}}U_{n} + \alpha_{n+1}\paren{M_{n+1}+\xi_{n+1}}} \\
&= \qty(1-\alpha_{n+1})z_{n} + \alpha_{n+1}h(x_{n}) \\
&= z_{n} + \alpha_{n+1}\paren{h(z_{n})-z_n + e_{n+1}}
\end{align}
which matches the form of \eqref{eq inexact km update} with $e_n = h\paren{x_{n-1}}-h\paren{z_{n-1}}$. 
Due to the non-expansivity of $h$ from \eqref{eq:h nonexpansive}, we have
\begin{equation}
    \norm{e_{n+1}} = \norm{h\qty(x_{n}) - h\qty(z_{n})} \leq \norm{x_{n} - z_{n}} = \norm{U_{n}} 
\end{equation}
The convergence of $x_n$ then follows directly from Lemma \ref{lem:bravo_2.1} which gives $\lim_{n \rightarrow \infty} z_n = x_*$ for some $x_* \in \fX_*$, and therefore $\lim_{n \rightarrow \infty}x_n =\lim_{n \rightarrow \infty} z_n + U_n = x_*$.
We note that here $e_n$ is stochastic while the~\eqref{eq inexact km update} result in Lemma~\ref{lem:bravo_2.1} considers deterministic noise.
This means we apply Lemma~\ref{lem:bravo_2.1} for each sample-path.
\end{proof}
\section{Additional Lemmas from Section \ref{sec:c_rate}}
\begin{corollary}\label{lem: M rate}
We have
\begin{equation}
    \E\qty[\norm{\Mb_n}]\leq C_{\ref{lem: M rate}}\tau_n\sqrt{\alpha_{n+1}}
\end{equation}
where $C_{\ref{lem: M rate}}$ is a deterministic constant.
\end{corollary}
\begin{proof}
    Starting from \eqref{eq: M stop} from Lemma \ref{lem:sup_M} to avoid redundancy, we directly have
    \begin{align}
        \E\left[\norm{ \Mb_n} \right] &= \fO\qty(\sqrt{\sum_{i=1}^{n} \alpha_{i,n}^2 \tau_{i}^2}).
    \end{align}
    Additionally, by Lemma \ref{lem:bravo_b.2}, we have $\sqrt{\sum_{i=1}^{n} \alpha_{i,n}^2 \tau_{i}^2} \leq \tau_n\sqrt{\alpha_{n+1}}$. Therefore, there exists a deterministic constant such that the corollary holds.
\end{proof}

\begin{corollary}\label{lem: e2 rate}
We have
\begin{equation}
    \E\qty[\norm{\eb{2}_n}]\leq C_{\ref{lem: e2 rate}}\alpha_n \tau_n
\end{equation}
where $C_{\ref{lem: e2 rate}}$ is a deterministic constant.
\end{corollary}
\begin{proof}
Starting from \eqref{eq: e2 bar stop} to avoid repetition, we have,
    \begin{align}
        \norm{\eb{2}_n}
        &\leq \alpha_n \norm{\nu \paren{x_{n},Y_{n+1}}} + \sum_{i=1}^n \left|\alpha_{i-1,n}-\alpha_{i,n} \right|\norm{\nu \paren{x_{i-1},Y_i}}.
    \end{align}
    Now we can take the expectation and apply the sample-path independent bound from Lemma \ref{lem:v_norm} with,
    \begin{align}
        \E \qty[\norm{\eb{2}_n}] &\leq C_{\ref{lem:v_norm}} \paren{\alpha_n\tau_n + \sum_{i=1}^n \left|\alpha_{i-1,n}-\alpha_{i,n} \right|\tau_{i-1}} \explain{Lemma \ref{lem:v_norm}}\\
        &= C_{\ref{lem:v_norm}} \paren{\alpha_n\tau_n  + \sum_{k=0}^{n-1} \left|\alpha_{k,n}-\alpha_{k+1,n} \right|\tau_{k}}
    \end{align}
    Lemma $\ref{lem:lr}$ and $\tau_k$ being monotonically increasing for $k\geq 1$ yields,
    \begin{align}
        \E \qty[\norm{\eb{2}_n}] &\leq  C_{\ref{lem:v_norm}} \paren{\alpha_n\tau_n + \alpha_{1,n}\tau_0 + \tau_{n}\sum_{k=1}^{n-1} \paren{\alpha_{k+1,n} - \alpha_{k,n}}}, \\
        &=  C_{\ref{lem:v_norm}} \paren{\alpha_n\tau_n + \alpha_{1,n} + \tau_{n}\paren{\alpha_{n,n} - \alpha_{1,n}}}, \explain{$\tau_0 \doteq 1$}\\
        &= \fO \qty( \alpha_n\tau_n).\explain{Lemma \ref{lem:bravo b1}}
    \end{align}
Therefore, there exists a deterministic constant we denote as $C_{\ref{lem: e2 rate}}$ such that
\begin{align}
    \E \qty[\norm{\eb{2}_n}] &\leq C_{\ref{lem: e2 rate}} \alpha_n \tau_n.
\end{align}
\end{proof}

\begin{corollary}\label{lem: e3 rate}
We have
\begin{equation}
    \E\qty[\norm{\eb{3}_n}] \leq C_{\ref{lem: e3 rate}}\alpha_n \sum_{i=1}^n \alpha_i \tau_i.
\end{equation}
\end{corollary}
\begin{proof}
Starting with \eqref{eq: e3 stop} from Lemma \ref{lem:sup_e3} to avoid redundancy, we have
\begin{align}
    \norm{\eb{3}_n} &\leq C_{\ref{lem:v Lipschitz}} \sum_{k=1}^n \alpha_{k,n}\alpha_{k}\qty(2C_{\ref{lem:xn_norm}}\tau_{k-1} +2\sum_{i=1}^{k-1} \alpha_{i}\norm{\e{1}_{i}} +C_H + \norm{\e{1}_k}).
\end{align}
Taking the expectation gives,
\begin{align}
    \E\qty[\norm{\eb{3}_n}] &\leq C_{\ref{lem:v Lipschitz}}\sum_{k=1}^n \alpha_{k,n}\alpha_{k}\qty(2C_{\ref{lem:xn_norm}}\tau_{k-1} +2\sum_{i=1}^{k-1} \alpha_{i}\E\qty[\norm{\e{1}_{i}}] +C_H + \E \qty[\norm{\e{1}_k}]).
\end{align}
Recall that $\tau_k$ is monotonically increasing. Additionally, by Lemma \ref{lem: e1 lr}, $\sum_{i=1}^{k-1} \alpha_{i}\E\qty[\norm{\e{1}_{i}}]$ converges and $\lim_{k\rightarrow \infty} \E \qty[\norm{\e{1}_k}]= 0$. Therefore, there exists a deterministic constant $C_{\ref{lem: e3 rate}}$ such that
\begin{align}
    \E\qty[\norm{\eb{3}_n}]&\leq C_{\ref{lem: e3 rate}} \sum_{k=1}^n \alpha_{k,n}\alpha_{k}\tau_{k-1}, \\
    &\leq C_{\ref{lem: e3 rate}} \alpha_n \sum_{i=1}^n \alpha_i \tau_i \explain{Lemma \ref{lem:bravo b1}}.
\end{align}
\end{proof}

\begin{lemma} \label{lem: E Un bound order}
For $\omega_n$ defined in \eqref{eq: expected un bound}, we have
\begin{align}
    \omega_n = \fO(\tau_{n} \sqrt{\alpha_{n+1}}),
\end{align}
which is dominated by $1/ \sqrt{\tau_n}$.

\end{lemma}
\begin{proof}
    From \eqref{eq: expected un bound}, we have
    \begin{align}
       \omega_n \doteq \underbrace{C_{\ref{lem: M second moment}} \tau_n\sqrt{\alpha_{n+1}}}_{K_1} + \underbrace{\sum_{i=1}^n \alpha_{i,n} \E\qty[\norm{\e{1}_i}]}_{K_2}+\underbrace{C_{\ref{lem: e2 rate}}\alpha_n\tau_n}_{K_3} + \underbrace{C_{\ref{lem: e3 rate}} \alpha_n \sum_{i=1}^n \alpha_i \tau_i}_{K_4} 
    \end{align}
    To prove the Lemma, we will examine each of the four terms and prove they are $\fO(\tau_{n} \sqrt{\alpha_{n+1}})$. For $K_1$, this is trivial. For $K_2$, we first recall from Lemma \ref{lem:lr} that $\alpha_n = \fO(\frac{1}{n^b})$ and
\begin{align}
    \tau_n = \begin{cases} 
      \fO\qty(n^{1-b}) & \text{if} \quad \frac{4}{5}<b<1, \\
      \fO\qty(\log n) & \text{if} \quad b=1.
      \end{cases}
\end{align}
Then we have,
\begin{align}
    \tau_n \sqrt{\alpha_{n+1}} = \begin{cases} 
      \fO\qty(\frac{1}{n^{\frac{3}{2}b -1}}) & \text{if} \quad \frac{4}{5}<b<1, \\
      \fO\qty(\frac{\log n}{\sqrt{n}}) & \text{if} \quad b=1.
      \end{cases} \label{eq: tau_n sqrt alpha n order}
\end{align}
    
    Then by Lemma \ref{lem: e1 lr} we have 
    \begin{align}
        \sum_{i=1}^n \alpha_{i,n} \E\qty[\norm{\e{1}_i}] &\leq \alpha_n \sum_{i=1}^n \E\qty[\norm{\e{1}_i}], \explain{Lemma \ref{lem:bravo b1}}\\
        &= \fO \qty(\alpha_n \sum_{i=1}^n \frac{1}{\sqrt{i}}), \\
        & = \fO \qty(\alpha_n \sqrt{n}) \\
        &= \fO \qty(\frac{1}{n^b} \sqrt{n}), \\
        &= \fO \qty(\frac{1}{n^{b-1/2}})
    \end{align}
    Because we have $\frac{3}{2}b -1 \leq b-\frac{1}{2}$ for $b \in (\frac{4}{5}, 1]$, we can see from \eqref{eq: tau_n sqrt alpha n order}, that $K_2$ is dominated by $K_1$.
    
    For $K_3$, by Lemma \ref{lem:lr} we have,
    \begin{align}
        \alpha_n \tau_n = \begin{cases} 
      \fO\qty(\frac{1}{n^{2b-1}}) & \text{if} \quad \frac{4}{5}<b<1, \\
      \fO\qty(\frac{\log n}{n}) & \text{if} \quad b=1.
      \end{cases}
    \end{align}
    It is clear from \eqref{eq: tau_n sqrt alpha n order}, $K_3$ is dominated by $K_1$.

For $K_4$, for the case when $b=1$, we have
\begin{align}
    \alpha_n \sum_{i=1}^n \alpha_i \tau_i &\leq \alpha_n \tau_n \sum_{i=1}^n \alpha_i \explain{$\tau_n$ increasing}\\&= \fO \qty(\frac{\log n}{n} \sum_{i=1}^n \frac{1}{i}), \\
    &= \fO \qty(\frac{\log^2 n}{n}).
\end{align}

For the case when $\frac{4}{5}< b < 1$, we have 
\begin{align}
    \alpha_n \sum_{i=1}^n \alpha_i \tau_i &= \fO\qty(\frac{1}{n^b}\sum_{i=1}^n \frac{1}{i^{2b-1}})
\end{align}
which we can approximate by an integral,
\begin{align}
    \int_1^n \frac{1}{x^{2b-1}} \ dx = \fO \qty( n^{2-2b}).
\end{align}
Therefore,
\begin{align}
    \alpha_n \sum_{i=1}^n \alpha_i \tau_i &= \fO \qty( n^{2-3b}).
\end{align}

Combining our results from the two cases, we have for $K_4$ 
\begin{align}
    \alpha_n \sum_{i=1}^n \alpha_i \tau_i = \begin{cases} 
      \fO\qty(\frac{1}{n^{3b-2}}) & \text{if} \quad \frac{4}{5}<b<1, \\
      \fO\qty(\frac{\log^2 n}{n}) & \text{if} \quad b=1.
      \end{cases}
\end{align}
Comparing with $K_1$ in \eqref{eq: tau_n sqrt alpha n order}, since we have $\frac{3}{2}b - 1<3b-2$ for $b \in (\frac{4}{5}, 1)$, we can see that $K_4$ is dominated by $K_1$, thereby proving $\omega_n = \fO(\tau_{n} \sqrt{\alpha_{n+1}})$. From \eqref{eq: tau_n sqrt alpha n order}, its easy to see that $\omega_n$ is dominated by
\begin{align}
    \frac{1}{\sqrt{\tau_n}} = \begin{cases} 
      \fO \paren{1/\sqrt{n^{1-b}}} & \text{if} \ \frac{4}{5}<b<1, \\
      \fO \paren{1/\sqrt{\log n}} & \text{if} \  b=1.
   \end{cases}
\end{align}

\end{proof}

\begin{lemma} \label{lem: bravo combo 2.11 3.1}
\begin{align}
    \sum_{k=2}^n 2\alpha_k \sigma\qty(\tau_n - \tau_k)\E\qty[\norm{U_{k-1}}] = \fO(1/\sqrt{\tau_n}).
\end{align}
\end{lemma}
\begin{proof}
    The proof of this lemma is a straightforward combination of the existing results of Theorems 2.11 and 3.1 from \citet{bravo2024stochastic}. First, from \eqref{eq: expected un bound}, we have
    \begin{align}
        \sum_{k=2}^n 2\alpha_k \sigma\qty(\tau_n - \tau_k)\E\qty[\norm{U_{k-1}}] \leq \sum_{k=2}^n 2\alpha_k \sigma\qty(\tau_n - \tau_k)\omega_{k-1}.
    \end{align}
    In the proof of Theorem 2.11 of \citet{bravo2024stochastic}, they prove that if there exists a decreasing convex function $f: (0,\infty) \rightarrow (0,\infty)$ of class $C^2$, and a constant $\gamma \geq 1$, such that for $k \geq 2$,

    \begin{align}
        \begin{cases}
            \omega_{k-1} \leq (1-\alpha_k)f(\tau_k), \\
            \alpha_k(1-\alpha_k) \leq \gamma \alpha_{k+1}(1-\alpha_{k+1}), \label{eq: conditions}
        \end{cases}
    \end{align}
    then
    \begin{align}
        \sum_{k=2}^n 2\alpha_k \sigma\qty(\tau_n - \tau_k)\omega_{k-1} \leq \frac{2\gamma}{\sqrt{\pi}} \int_{\tau_1}^{\tau_n} \frac{f(x)}{\sqrt{\tau_n - x}}dx + 2 \alpha_n \omega_{n-1}. \label{eq: conv int}
     \end{align}
    
    Theorem 3.1 in \citet{bravo2024stochastic} establishes that for the step size schedule specified in Assumption \ref{as:lr}, 
    there exist constants $ \gamma \geq 1 $ and a function $ f(x) $ satisfying \eqref{eq: conditions} with $\omega_n = \fO(\tau_n \sqrt{\alpha_{n+1}})$. Specifically, they show with
\begin{align}
    f(x) =
    \begin{cases}
        \kappa x(1+x)^{-b/2(1-b)} & \text{if } b < 1, \\
        \kappa xe^{-x/2} & \text{if } b = 1,
    \end{cases}
\end{align}
for some constant $ \kappa $ and $\gamma = \frac{32}{27}$, \eqref{eq: conditions} is satisfied. Moreover, they demonstrate that the resulting convolution integral in \eqref{eq: conv int} evaluates to $ \fO(1/\sqrt{\tau_n}) $.

Combining these results with Lemma \ref{lem: E Un bound order} which shows that $\omega_n$ is dominated by $1 / \sqrt{\tau_n}$, the right-hand side of \eqref{eq: conv int} simplifies to $ \fO(1/\sqrt{\tau_n}) $, which completes the proof. For detailed steps, we refer the reader to \citet{bravo2024stochastic}.
\end{proof}

\section{Extension of Theorem 2.1 of \citet{borkar2009stochastic}}\label{sec:borkar_ext}
In this section, we present an extension of Theorem 2 from \citep{borkar2009stochastic} for completeness.
Readers familiar with stochastic approximation theory should find this extension fairly straightforward. 
Originally, Chapter 2 of \citep{borkar2009stochastic} considers stochastic approximations of the form,
\begin{align}
    y_{n+1} = y_n + \alpha_n\paren{h(y_n) + M_{n+1}+ \xi_{n+1}} \label{eq:borkar_iter}
\end{align}
where it is assumed that $\xi_{n} \rightarrow 0$ almost surely. 
However, our work requires that we remove the assumption that $\xi_n \rightarrow 0$, and replace it with a more mild condition on the asymptotic rate of change of $\xi_n$, akin to \citet{kushner2003stochastic}.
\begin{assumption}\label{as:borkar_noise}
For any $T>0$,
\begin{align}
    \label{eq: arc}
    \lim_{n\to\infty} \sup_{n\leq j \leq m(n, T)}\norm{\sum_{i=n}^{j}\alpha_i \xi_i} = 0 \qq{a.s.}
\end{align}
where $m(n, T) \doteq \min\qty{k | \sum_{i=n}^k \alpha(i) \geq T}$.
\end{assumption}

The next four assumptions are the same as the remaining assumptions in Chapter 2 of \cite{borkar2009stochastic}.

\begin{assumption}\label{as:borkar_lipschitz}
The map $h$ is Lipschitz: $\norm{h\paren{x} - h\paren{y}} \leq L\norm{x-y}$ for some $0<L<\infty$.
\end{assumption}

\begin{assumption}\label{as:borkar_lr}
The step sizes $\qty{\alpha_n}$ are positive scalars satisfying
\begin{align}
\sum_n \alpha_n = \infty, \sum_n \alpha_n^2 < \infty
\end{align}
\end{assumption}

\begin{assumption}\label{as:borkar_mds}
$\qty{M_n}$ is a martingale difference sequence w.r.t the increasing family of $\sigma$-algebras
\begin{align}
\fF_n \doteq \sigma \qty(y_m, M_m, m\leq n) = \sigma \qty(y_0, M_1,\ldots, M_n), \, n\geq 0.
\end{align}
That is,
\begin{align}
    \E\left[M_{n+1}| \fF_n\right] = 0 \qq{a.s.}, n\geq 0.
\end{align}
Furthermore, $\qty{M_n}$ are square-integrable with
\begin{align}
    \E\left[\norm{M_{n+1}}^2| \fF_n\right] \leq K\paren{1+ \norm{x_n}^2} \ \qq{a.s.}, \ n\geq 0,
\end{align}
for some constant $K>0$
\end{assumption}

\begin{assumption}\label{as:borkar_sup}
The iterates of \eqref{eq:borkar_iter} remain bounded almost surely, i.e.,
\begin{align}
\sup_n \norm{y_n} < \infty
\end{align}
\end{assumption}

\begin{theorem}[Extension of Theorem 2.1 from \cite{borkar2009stochastic}] \label{thm:borkar_thm2}
    Let Assumptions \ref{as:borkar_noise}, \ref{as:borkar_lipschitz},  \ref{as:borkar_lr}, \ref{as:borkar_mds}, \ref{as:borkar_sup} hold. Almost surely, the sequence $\qty{y_n}$ generated by \eqref{eq:borkar_iter} converges to a (possibly sample-path dependent) compact connected internally chain transitive set of the ODE
    \begin{align}
\label{eq:Un_ode}
        \dv{y(t)}{t} = h(y(t)).
    \end{align}
\end{theorem}

\begin{proof}
 We now demonstrate that even with the relaxed assumption on $\xi_n$,
 we can still achieve the same almost sure convergence of the iterates achieved by \cite{borkar2009stochastic}. 
    Following Chapter 2 of \cite{borkar2009stochastic}, we construct a continuous interpolated trajectory $\bar{y}\paren{t}, t\geq 0$, and show that it asymptotically approaches the solution set of \eqref{eq:Un_ode} almost surely. Define time instants $t\paren{0}=0, t\paren{n}= \sum_{m=0}^{n-1} \alpha_m, n\geq 1$. By assumption \ref{as:borkar_lr}, $t\paren{n}\uparrow \infty$. Let $I_n \doteq \left[t\paren{n}, t\paren{n+1}\right], n\geq 0$. Define a continuous, piece-wise linear $\bar{y}\paren{t}, t\geq0$ by $\bar{y}{\paren{t\paren{n}}} = y_n, \ n\geq 0$, with linear interpolation on each interval $I_n$: 
    \begin{equation}
    \bar{y}\paren{t} = y_n + \paren{y_{n+1} - y_n}\frac{t - t\paren{n}}{t\paren{n+1} - t\paren{n}}, t \in I_n \label{eq:borkar2_interp_traj}
    \end{equation}
    It is worth noting that $\sup_{t\geq 0} \norm{\bar{y}\paren{t}} = \sup_{n}\norm{y_n} < \infty$ almost surely by Assumption \ref{as:borkar_sup}. Let $y^s\paren{t}, t\geq s,$ denote the unique solution to $\eqref{eq:Un_ode}$ `starting at s': 
    \begin{align}
    \dv{y^s(t)}{t} &= h\paren{y^s\paren{t}}, t\geq s,
    \end{align}
    with $y^s\paren{s} = \bar{y}\paren{s}, s\in \R$. Similarly, let $y_s\paren{t}, t\geq s,$ denote the unique solution to $\eqref{eq:Un_ode}$ `ending at s':
    \begin{align}
    \dv{y_s(t)}{t} &= h\paren{y_s\paren{t}}, t\leq s,
    \end{align}
    with $y_s\paren{s} = \bar{y}\paren{s}, s\in \R$.
    Define also 
    \begin{align}
        \zeta_n &= \sum_{m=0}^{n-1} \alpha_m \paren{M_{m+1} + \xi_{m+1}}, \ n\geq 1 \label{eq:zeta_n_define}
    \end{align}

After invoking Lemma \ref{thm:borkar_2_thm1}, the analysis and proof presented for Theorem 2 in \cite{borkar2009stochastic} applies directly, yielding our desired extended result.
\end{proof}

\begin{lemma}[Extension of Theorem 1 from \cite{borkar2009stochastic}] \label{thm:borkar_2_thm1}
    Let \ref{as:borkar_noise} $-$ \ref{as:borkar_sup} hold. We have for any $T>0$,
    \begin{align}
        \lim_{s\rightarrow \infty} &\sup_{t\in \left[s,s+T\right]} \norm{\bar{y}\paren{t} - y^s\paren{t}} = 0, \qq{a.s.} \\
        \lim_{s\rightarrow \infty} &\sup_{t\in \left[s,s+T\right]} \norm{\bar{y}\paren{t} - y_s\paren{t}} = 0, \qq{a.s.}
    \end{align}
\end{lemma}
\begin{proof}
    Let $t\paren{n+m}$ be in $\left[t(n), t(n)+T\right]$. Let $\left[t\right] \doteq \max \qty{t(k) : t(k) \leq t}$. Then,
    \begin{align}
        \bar{y}\paren{t\paren{n+m}} = \bar{y}\paren{t\paren{n}} + \sum_{k=0}^{m-1}\alpha_{n+k}h\paren{\bar{y}\paren{t\paren{n+k}}} + \delta_{n,n+m} 
        \explain{2.1.6 in \text{\cite{borkar2009stochastic}}} \,\,\,\,
        \label{eq:borkar_2.1.6}
    \end{align}
    where $\delta_{n,n+m} \doteq \zeta_{n+m} - \zeta_n$. 
    \citet{borkar2009stochastic} then compares this with 
    \begin{align}
        y^{t(n)}\paren{t\paren{n+m}} = \bar{y}\paren{t\paren{n}} + \sum_{k=0}^{m-1}\alpha_{n+k}h\paren{y^{t(n)}\paren{t\paren{n+k}}} \\ \ + \int_{t(n)}^{t(n+m)}\paren{h\paren{y^{t(n)}\paren{z}} - h\paren{y^{t(n)}\paren{\left[z\right]}}}dz. \explain{2.1.7 in \text{\cite{borkar2009stochastic}}}
    \end{align}

    Next, \citet{borkar2009stochastic} bounds the integral on the right-hand side by proving 
    \begin{align}
        \norm{\int_{t(n)}^{t(n+m)}\paren{h\paren{y^{t(n)}\paren{t}} - h\paren{y^{t(n)}\paren{\left[t\right]}}}dt} \leq C_TL\sum_{k=0}^\infty \alpha_{n+k}^2 \xrightarrow{n\uparrow \infty} 0, \ \qq{a.s.} \explain{2.1.8 in \text{\cite{borkar2009stochastic}}}
    \end{align}
    where $C_T \doteq \norm{h\paren{0}} + L\paren{C_0 + \norm{h\paren{0}}T}e^{LT} < \infty $ almost surely and $C_0 \doteq \sup_n \norm{y_n} < \infty$ a.s. by Assumption \ref{as:borkar_sup}.

    Then, we can subtract (2.1.7) from (2.1.6) and take norms, yielding
    \begin{align}
        \norm{\bar{y}\paren{t\paren{n+m}}- y^{t(n)}\paren{t\paren{n+m}}} &\leq L \sum_{i=0}^{m-1}\alpha_{n+i}\norm{\bar{y}\paren{t\paren{n+i}}- y^{t(n)}\paren{t\paren{n+i}}} \\
        &\quad + C_TL\sum_{k\geq 0} \alpha_{n+k}^2 + \sup_{0\leq k \leq m(n,T)}\norm{\delta_{n,n+k}}. \label{eq:our_subtraction}
    \end{align}
    The key difference between \eqref{eq:our_subtraction} and the analogous equation in \citet{borkar2009stochastic} Chapter 2, is that we replace the $\sup_{k \geq 0}$ with a $\sup_{0\leq k \leq m(n,T)}$. The reason we can make this change is that we defined $t(n+m)$ to be in the range $\left[t(n), t(n)+T\right]$. Recall that we also defined $m(n, T) \doteq \min\qty{k | \sum_{i=n}^k \alpha(i) \geq T}$ in Assumption \ref{as:borkar_noise}, so we therefore know that $m \leq m(n,T)$ in \eqref{eq:borkar_2.1.6}. \citet{borkar2009stochastic} unnecessarily relaxes this for notation simplicity, but a similar argument can be found in \cite{kushner2003stochastic}. 
    
    Also, we have,
    \begin{align}
        \norm{\delta_{n,n+k}} &= \norm{\zeta_{n+k} - \zeta_{n}}, \\
        &= \norm{\sum_{i=n}^k \alpha_i \paren{M_{i+1}+\xi_{i+1}}}, \explain{by \eqref{eq:zeta_n_define}}\\
        &\leq \norm{\sum_{i=n}^k \alpha_i M_{i+1}}+ \norm{\sum_{i=n}^k \alpha_i \xi_{i+1}}.
    \end{align}
    \citet{borkar2009stochastic} proves that $\paren{\sum_{i=0}^{n-1} \alpha_i M_{i+1}, \fF_n}, \ n\geq 1$ is a zero mean, square-integrable martingale. By \ref{as:borkar_lr}, \ref{as:borkar_mds}, \ref{as:borkar_sup},
    \begin{align}
        \sum_{n\geq 0} \E \left[\norm{\sum_{i=0}^{n} \alpha_i M_{i+1} - \sum_{i=0}^{n-1} \alpha_i M_{i+1}} \, \bigg| \, \fF_n\right] = \sum_{n\geq 0} \E \left[\norm{M_{n+1}}^2 \, | \, \fF_n\right] < \infty.
    \end{align}
    
    Therefore, the martingale convergence theorem gives the almost sure convergence of $\paren{\sum_{i=n}^k \alpha_i M_{i+1}, \fF_n}$ as $n \rightarrow \infty$. Combining this with assumption \ref{as:borkar_noise} yields,
    \begin{align}
        \lim_{n \rightarrow \infty} \sup_{0 \leq k \leq m(n,T)} \norm{\delta_{n,n+k}} = 0 \qq{a.s.}
    \end{align}
    Using the definition of $K_{T,n} \doteq C_TL\sum_{k\geq 0} \alpha_{n+k}^2 + \sup_{0\leq k \leq m(n,T)}\norm{\delta_{n,n+k}}$ given by \cite{borkar2009stochastic}, we have proven that our slightly relaxed assumption still yields $K_{T,n} \rightarrow 0$ almost surely as $n \rightarrow \infty$. The rest of the argument for the proof of the theorem in \citet{borkar2009stochastic} holds without any additional modification.
\end{proof}

\end{document}